\documentclass[11pt, a4paper, logo, twocolumn, copyright, nonumbering]{googledeepmind}

\usepackage{url}            
\usepackage{booktabs}       
\usepackage{amsfonts}       
\usepackage{nicefrac}       
\usepackage{microtype}      
\usepackage{amssymb}
\usepackage{amsthm}
\usepackage{comment}
\usepackage{mathtools}
\usepackage{multirow}
\usepackage{algorithmic}
\usepackage[linesnumbered,ruled,vlined]{algorithm2e}
\usepackage{xcolor}         
\usepackage{subfigure}  
\usepackage[authoryear, sort&compress, round]{natbib}
\title{Evaluating Agents using Social Choice Theory}

\correspondingauthor{Marc Lanctot, {\tt lanctot@google.com}.}

\paperurl{}

\reportnumber{} 

\renewcommand{\today}{2023-12-11}

\author[1]{Marc Lanctot}
\author[1,2]{Kate Larson}
\author[1]{Yoram Bachrach}
\author[1]{Luke Marris}
\author[3]{Zun Li}
\author[1]{Avishkar Bhoopchand}
\author[1]{Thomas Anthony}
\author[4]{Brian Tanner}
\author[1]{Anna Koop}

\affil[1]{Google DeepMind}
\affil[2]{University of Waterloo}
\affil[3]{University of Michigan}
\affil[4]{Artificial.Agency}

\newcommand{\argmax}{\operatornamewithlimits{argmax}}

\newcommand{\cG}{\mathcal{G}}

\newcommand{\defword}[1]{\textbf{\boldmath{#1}}}
\newcommand{\ie}{{\it i.e.},~}  
\newcommand{\eg}{{\it e.g.},~}  

\newtheorem{theorem}{Theorem}
\newtheorem{lemma}{Lemma}

\newtheorem{definition}{Definition}

\begin{abstract}
We argue that many general evaluation problems can be viewed through the lens of voting theory. Each task is interpreted as a separate voter, which requires only ordinal rankings or pairwise comparisons of agents to produce an overall evaluation. By viewing the aggregator as a \emph{social welfare function}, we leverage centuries of research in social choice theory to derive principled evaluation frameworks with axiomatic foundations. These evaluations are interpretable and flexible, while avoiding many of the problems currently facing cross-task evaluation.
We apply this Voting-as-Evaluation (VasE) framework across multiple settings, including
reinforcement learning, large language models, and humans.
In practice, we observe that VasE can be more robust than popular evaluation frameworks (Elo and Nash averaging),
discovers properties in the evaluation data not evident from scores alone,
and can predict outcomes better than Elo in a complex seven-player game.
We identify one particular approach, maximal lotteries,
that satisfies important consistency properties relevant to evaluation, is computationally efficient (polynomial in the size of the evaluation data), and identifies game-theoretic cycles.
\end{abstract}

\begin{document}

\maketitle

\section{Introduction}

In this paper, we bridge two sub-disciplines of artificial intelligence: evaluation of general agents and computational social choice.
Evaluation is central to driving progress in artificial intelligence and machine learning. Without principled evaluation procedures it is challenging, if not impossible, for both researchers and the broader public to understand the strengths and limitations of different systems. 
This confusion is further exacerbated when considering the evaluation problem across multiple tasks, features and metrics.
For example, 
the Arcade Learning Environment (ALE)~\citep{bellemare13arcade} inspired the development of general agents via deep reinforcement learning (RL) approaches, but the breadth of options available for assessment gave rise to a diverse set of incompatible evaluation methodologies~\citep{machado18arcade}.
With the proliferation and accessibility of general environment suites
and increasing interaction with human agents, there is a need for robust and principled evaluation schemes for general agents.

\begin{figure}[t]
\begin{center}
\includegraphics[scale=1.2]{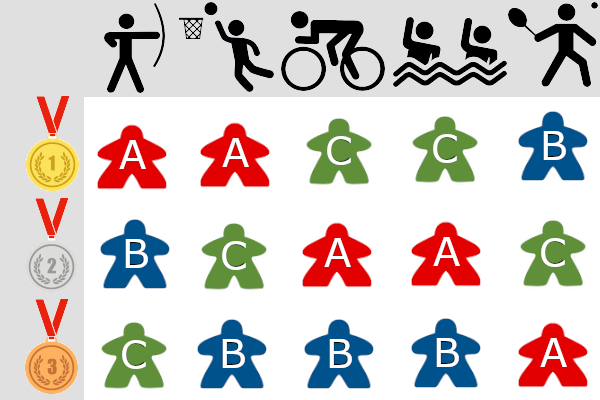}
\end{center}
\caption{An evaluation problem for general agents. Each event has its own separate metric for determining the ranking of each participant, with gold, silver, and bronze corresponding to first, second, and third place ranks.}
\label{fig:meeple_pentathlon}
\end{figure}
To ground our comparison of evaluations methods, we will use the pentathlon 
depicted in Figure~\ref{fig:meeple_pentathlon} where agents $A$, $B$, and $C$ compete in five different events: archery, basketball, cycling, synchronized swimming, and tennis. 
Each sport has its own specific evaluation criteria: \eg fastest time, most points, match outcome.
Every sport also awards its  top winners with gold, silver, and bronze medals.
Who should be considered the winner of the pentathlon?
Social choice theory provides a framework for answering this question 
when we consider the standings within each competitions as votes aggregated according to some voting scheme. 
Though the voting methods of social choice theory have long been extended beyond the field's political origins, they have yet to be widely used for the evaluation of machine learning agents.

How do researchers currently evaluate agents?
The most commonly used system to compare agents is the Elo rating system~\citep{Elo78}. Elo was designed originally to rank chess agents. 
It is still widely-used today in spite of several well-known deficiencies~\citep{Balduzzi18ReEval}. 
For example, it cannot represent non-transitive relationships between agents, as in the classic game of Rock, Paper, Scissors where rock beats scissors, scissors beats paper, and paper beats rock. 
Because of the cyclic relationship, Elo assigns the same rating to each agent, predicting that one beats the other with probability $\frac{1}{2}$. 
Elo can also be manipulated by cloning: the addition of multiple copies of similar agents can overturn rankings.
A technique such as Nash averaging~\citep{Balduzzi18ReEval} addresses the cloning issue, but raises others. 
Nash averaging is a game-theoretic approach that interprets evaluation data as a game between agent and task (or agent versus agent), ranking agents by their expected utility against a maximum-entropy Nash equilibrium.
We identify three problems with Nash averaging: 
it requires scores be directly comparable, it aggressively filters out tasks (due to selecting tasks adversarially), and can be highly sensitive to the set of included agents.
In contrast, social choice theory provides methods that overcome these problems in a principled manner.

We introduce {\bf Voting-as-Evaluation (VasE)}: an evaluation system for general agents based on computational social choice.
Here, by ``general agent'' we mean any agent (human or learned) that can perform multiple tasks.
In particular, the paper contains the following novel contributions:
\begin{itemize}
  \item Presentation of a framework (VasE) for adapting computational social choice to the evaluation of general agents.
  \item A characterization of relevant consistency properties
  and their effects on the evaluation of general agents.
  \item Analysis of the findings of VasE on three different evaluation domains: reinforcement learning, large language models, and human competitions.
\item An iterative variant of the maximal lotteries voting scheme~\citep{Fishburn84} that produces a full ranking of agents: 
    a game-theoretic method that repeatedly solves two-player zero-sum margin games as represented by the voter margin matrix, such as depicted in Figure~\ref{fig:meeple_pentathlon_N_M_mats}.
\end{itemize}

\noindent VasE provides several benefits over existing evaluation schemes. 
VasE does not require score normalization across tasks, only ordinal rankings. This makes it straightforward to apply, even across incomparable metrics. 
Task distribution weights are user-controlled rather than system-controlled. This gives the evaluator the flexibility to decide the relative importance of each task depending on the context.
The aggregate outcome is interpretable, as the consistency properties of voting methods are designed to satisfy specific practical axioms. 
These axioms also provide guarantees, such as robustness to clones.
Finally, compared to Nash averaging, VasE does not filter tasks, and can be more robust to changes in subsets of agents.

Through our analysis, maximal lotteries emerges as a particularly useful voting method. Maximal lotteries satisfies all of the consistency properties discussed in this paper, is computable in polynomial time in the size of the evaluation data, and identifies groups of game-theoretic cycles (non-transitivity) among agents.   

\begin{figure}[t]
\begin{center}
~~~~~\includegraphics[scale=0.25]{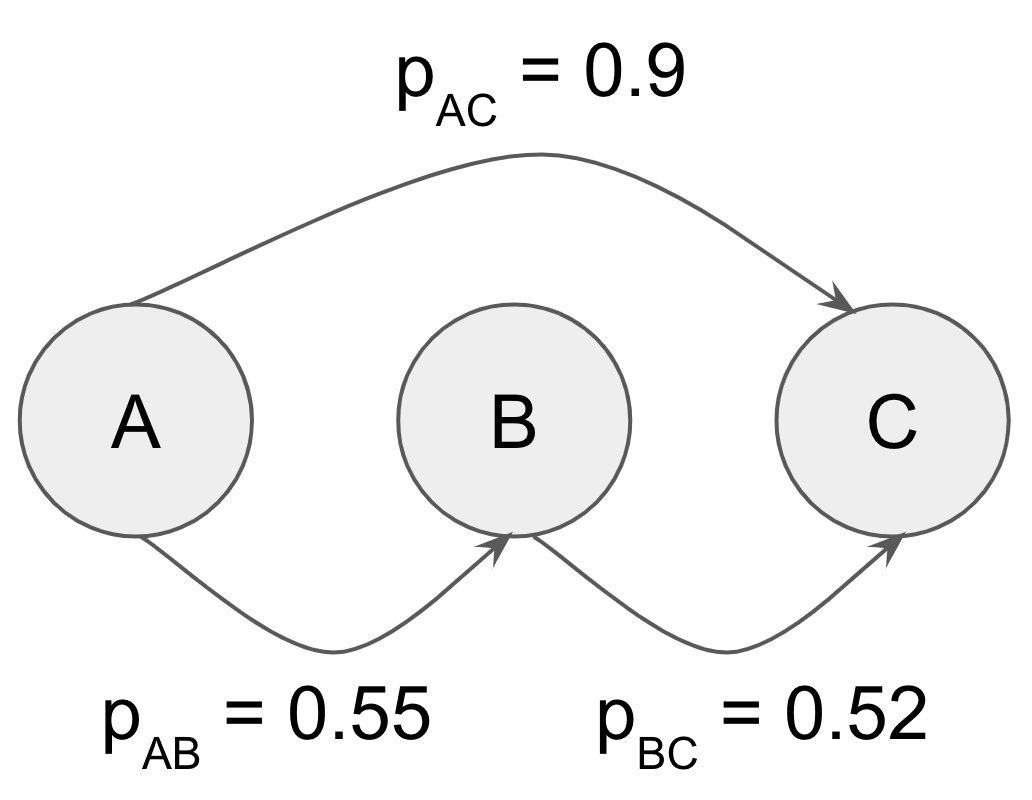}\\
\vspace{0.3cm}
\includegraphics[scale=0.35]{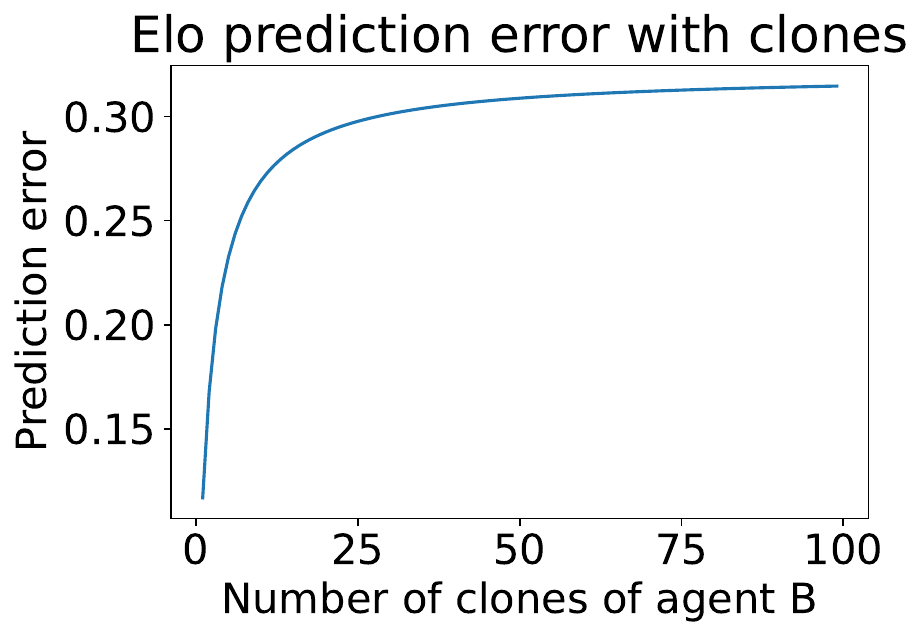}
\end{center}
\caption{Example effect of clones on Elo predictions. (Top) True win frequencies of a transitive relationship: $A$ beats $B$ beats $C$. (Bottom) Worst-case error of true $p_{AC}$ from Elo's predictions $\hat{p}_{AC}$ with clones of $B$.
\label{fig:elo-cloning}}
\end{figure}

Implementations of the voting methods used in VasE, example applications, and datasets can be found in OpenSpiel~\citep{LanctotEtAl2019OpenSpiel}\footnote{Specifically, in the following folder: \url{https://github.com/google-deepmind/open_spiel/tree/master/open_spiel/python/voting}.}.

\section{Evaluating Agents: Rating and Ranking}
\label{sec:bg-agent-eval}

Elo is a classic rating system that ranks skill using a simple logistic model learned from win/loss outcomes~\citep{Elo78}. 
A rating, $r_i$, is assigned to each player $i$ such that the probability of player $i$ beating player $j$ is predicted as $\hat{p}_{ij} = \frac{1}{1 + 10^{(r_j - r_i)/400}}$. 
While Elo was designed specifically to rank players in the two-player zero-sum, perfect information game of Chess, it has been widely applied to other domains\footnote{\url{https://en.wikipedia.org/wiki/Elo_rating_system\#Use_outside_of_chess}.} including in evaluation of large language models~\citep{zheng2023judging}.
TrueSkill~\citep{TrueSkill06} and bayeselo~\citep{Coulom05bayeselo} are rating systems based on similar foundations (Bradley-Terry models of skill) that also model uncertainty over ratings using Bayesian methods.

\subsubsection{Drawbacks of Elo}

Elo is simple, but it is well-known that it cannot model nontransitive relationships~\citep{Balduzzi18ReEval}. 
Elo predictions can be incorrect even on transitive games~\citep{Nihar17Simple,bertrand2023on}. 
Its inability to model nontransitivity is problematic since real-world games have large portions of the strategy space that are ``extremely non-transitive''~\citep{Czarnecki20RealWorld}. In practice, this transitivity assumption is violated in Go agents~\citep{Tuyls18Generalized} and in human play data in Chess~\citep{Sanjaya22Chess}, the very game Elo was designed for. 

\begin{figure}[t]
\centering
\includegraphics[scale=0.25]{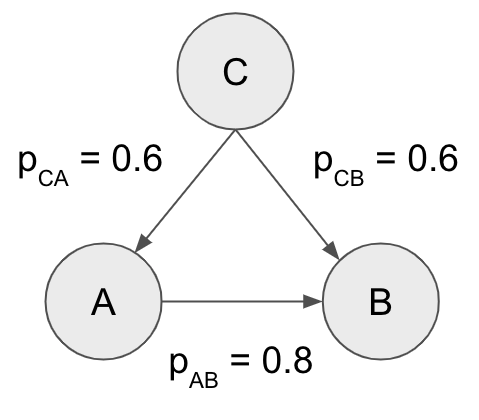}
\caption{Head-to-head win rates for agents in the pentathlon from Figure~\ref{fig:meeple_pentathlon}. Given these win rates, Elo assigns A and C the same rating.}
\label{fig:elo-meeple-winrates}
\end{figure}

Elo ratings can also be manipulated by cloned agents, which formally are agents that perform similarly (though not necessarily identically) to each other.
This is illustrated in  Figure~\ref{fig:elo-cloning}.
Denote $r_A$ as player $A$'s Elo rating and  $p_{AB}$ is the true probability (win rate) of agent $A$ beating $B$. 
Suppose there are three agents with true win rates of $p_{AB} = 0.55, p_{BC} = 0.52, \mbox{ and } p_{AC} = 0.9$.
Let $\vec{e} = ( p_{ij} - \hat{p}_{ij} )_{(i, j)}$ be a vector of differences between true win rates (ground truth) and Elo's predictions, and the overall prediction error as $|\vec{e}|_{\infty} = \max_{i,j}|e_{ij}|$.
Elo's prediction error can very quickly rise as clones are introduced. E.g. if there are 10 copies of $B$, then Elo predicts $r_A - r_C = 98$ and there is a significant error: $\hat{p}_{AC} = 0.631$ when the ground truth $p_{AC} = 0.9$.

The win rates for our pentathlon are shown in Figure~\ref{fig:elo-meeple-winrates}, which illustrates another surprising situation for Elo. 
Suppose we use Elo to rank agents from the example in Figure~\ref{fig:meeple_pentathlon}, where every instance of $x$ being ranked above $y$ counts as a win for $x$ and loss for $y$. 
In this example, despite agent $C$ winning the majority of the time against both $A$ and $B$, and $A$ only beating one of the other agents, Elo assigns the same rating to $C$ and $A$. In the appendix, we formally show that Elo's update rule leads to the same magnitude of the gradient components for $r_A$ and $r_C$ on all steps.

\subsubsection{Game-Theoretic Invariant Evaluation}

\citet{Balduzzi18ReEval} use the Schur decomposition to separate the problem of agent comparison into cyclic and transitive components, proposing multi-dimensional Elo (mElo) and Nash averaging. 
In the Agent-vs-Task (AvT) setting, they use a matrix $S$ with $m$ rows (agents) and $n$ columns (tasks) with values corresponding to scores (\eg \emph{normalized} average reward) for each agent in each task. In the Agent-vs-Agent (AvA) setting (\eg Chess), both rows and columns are agents and payoffs are expected utility for each agent. 
Agents are ranked by their expected values against a maximum-entropy Nash equilibrium strategy of a two-player zero-sum game.
mElo can model nontransitive relationships but can be manipulated by adding copies of agents (i.e. it is not clone-proof).

Nash averaging achieves clone-proofness for both agents and tasks due to properties of Nash equilibria. 
However, these properties allow partial support to be assigned to the task distribution, which we call {\it adversarial task selection}.
As a result, agents may be assessed over just a few tasks despite a wealth of available data. 

\citet{Balduzzi18ReEval} argue that an evaluation method should also be invariant to redundant copies of tasks.
This task invariance property is, however, highly constraining.
We show some counter-intuitive effects arising from it in our empirical evaluation section.
Furthermore, being invariant to copies of tasks means there is no mechanism to weight the importance of tasks, which is often be a useful feature. 
Thus, we do not address frameworks that are invariant in this sense.

Other game-theoretic solution concepts can be used in place of Nash equilibrium: for example, correlated equilibrium in the $n$-player general AvA setting~\citep{marris2022game}. Jordan et al. propose an evaluation scheme for general RL agents~\citep{Jordan20Evaluating} using  $\alpha$-Rank~\citep{Omidshafiei19AlphaRank} in a similar game between agents and tasks: again, this requires score normalization.

\section{Voting as Evaluation of General Agents}
\label{sec:vase}

We propose viewing the evaluation of AI agents through the lens of social choice. 
There is a natural mapping of evaluation problems into voting schemes, and by carefully selecting appropriate voting schemes, evaluation procedures inherit properties that are desirable within the general evaluation context. We title this framework \defword{Voting as Evaluation} (VasE), and provide a brief introduction to formal concepts in the following section. More technical details are provided in the appendix.

\subsection{Social Choice Theory Basics}
\label{sec:bg-soccho}

\begin{figure}[t]
\begin{center}
\begin{tabular}{ccc}
\includegraphics[scale=1]{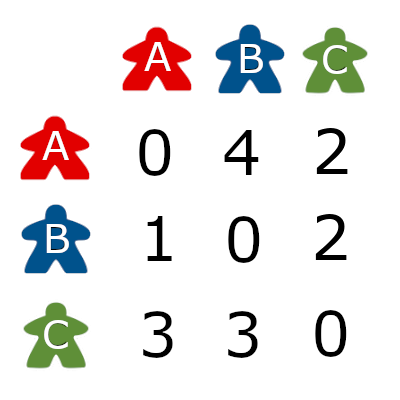} & ~ &
\includegraphics[scale=1]{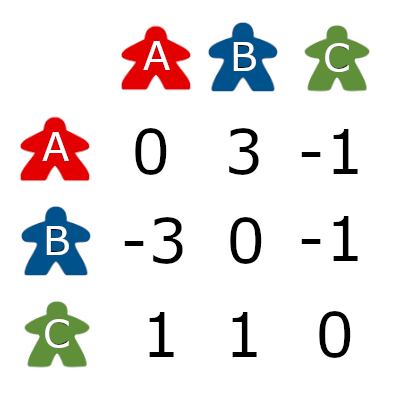} \\
\end{tabular}
\end{center}
\caption{(Left) The voter preference matrix $N(x, y)$ shows the number of events (votes) in which the agent on row $x$ is preferred to the agent on column $y$, for the example in Figure~\ref{fig:meeple_pentathlon}. 
(Right) The voter margin matrix whose entries are $M(x,y) = \delta(x,y) = N(x,y) - N(y,x)$.}
\label{fig:meeple_pentathlon_N_M_mats}
\end{figure}

A \defword{voting scheme} is defined as $\langle A, V, f\rangle$ 
where $A=\{a_1,\ldots, a_m\}$ is the set of \defword{alternatives},  
$V=\{v_1,\ldots, v_n\}$ is the set of \defword{voters}, 
and $f$ is the \defword{voting rule} that determines how votes are aggregated.
Voters have \defword{preferences} over alternatives: $a_1\succ_{v_i} a_2$ indicates voter $v_i$ strictly prefers alternative $a_1$ over alternative $a_2$, and $a_1\succeq_{v_i} a_2$ indicates the voter has a weak preference. 
These preferences induce (nonstrict) total orders over alternatives, which we denote by $\mathcal{L}$.  
A \defword{preference profile}, $[\succeq]\in \mathcal{L}^n$, is a vector specifying the preferences of each voter in $V$. 
It can be useful to summarize the preference profile in a \defword{voter preference matrix} $N$ or \defword{vote margin matrix} $M$.
The preference count $N(x,y)$, $x, y \in A$ is the number of voters in $[\succeq]$ that strictly prefer $x$ to $y$. 
The vote margin is the difference in preference count: $\delta(x, y) = N(x,y) - N(y,x)$. Both $N$ and $M$ matrices for the pentathlon example are shown in Figure~\ref{fig:meeple_pentathlon_N_M_mats}.

The central problem of social choice theory is how to aggregate preferences of a population so as to reach some collective decision.
A voting rule that determines the ``winner'' (a non-empty subset, possibly with ties), is a \defword{social choice function} (SCF). A voting rule that returns an aggregate ranking over all the alternatives is a \defword{social welfare function} (SWF).
Much of the social choice literature focuses on understanding what properties different voting rules support.

The \defword{Condorcet winner} formally defines a fairly intuitive concept: $A$ is the (strong) Condorcet winner if the number of votes where $A$ is ranked more highly than $B$ is greater than vice versa for all other alternatives $B$.
A weak Condorcet winner wins or ties in every head-to-head pairing.
Formally, given a  preference profile, $[\succeq]$, 
a weak Condorcet winner is an alternative $a^* \in A$ such that $\forall a' \in A, \delta(a^*, a') \ge 0$. 
If the inequality is strict for all pairs except  $(a^*, a^*)$ then we call it a strong Condorcet winner.
While many have argued that this definition captures the essence of the correct collective choice~\citep{deCondorcet1785}, 
in practice many preference profiles have no Condorcet winner. 

Condorcet consistent voting schemes are those that return a Condorcet winner when it exists, but differ on how they handle settings with no Condorcet winners.  We compare four deterministic Condorcet methods:
Kemeny-Young~\citep{Kemeny59,young1978consistent},
Copeland~\citep{Copeland50} due to its efficiency, and
ranked pairs~\citep{Tideman87} and Schulze~\citep{Schulze11} due to their clone-proofness.
See the appendix for formal definitions.

Scoring rules define another class of voting rules.  Scoring rules translate voter rankings to 
a score vector $\mathbf{w}=(w_1,\ldots, w_m)$ where $w_1\geq w_2\geq \ldots \geq w_m$ and $w_1>w_m$. 
The overall score for each alternative is computed by summing across all associated scores. 
Alternatives are ranked by their scores with the ``winning'' alternative being the one with the highest score.  Many familiar voting schemes belong to this class of voting rules, including plurality, Borda, and $k$-approval voting~\citep{Brandt16Handbook}.

All scoring rules satisfy a property called \defword{population consistency}~\citep{smith1973aggregation,young1975social}.
This  property  requires that when two disjoint sets of voters agree on some winning alternative $a^*$, merging the populations does not change the result. Scoring rules are not, however, Condorcet consistent. More broadly, population consistency and Condorcet consistency are incompatible~\citep{young1978consistent} for all deterministic voting schemes. Thus anyone selecting a voting scheme to use must make an explicit decision as to which property they value.

\subsubsection{Probabilistic Voting: Maximal Lotteries}

Probabilistic social choice functions (PSCFs)~\citep{Endriss17Trends} have emerged to address these impossibility results.
Given a preference profile, a PSCF returns a \emph{distribution} over alternatives (a lottery). One particular PSCF is the \defword{maximal lottery}~\citep{Krewaras65, Fishburn84}. 
A maximal lottery,  
$p \in \Delta(A)$, is one that is (weakly) preferred to any other lottery: namely $p^T M q \ge 0, \forall q \in \Delta(A)$, where $M$ is the voter margin matrix,
$M_{x,y}=(\delta(x,y))_{x,y\in A}$.
Equivalently, one can view $M$ as the payoffs of a carefully constructed symmetric zero-sum \defword{margin game} where the payoffs are win or loss magnitudes of different pairwise comparisons. The  maximal lottery is, thus,  the mixed maximin (or Nash equilibrium) solution to the game, and can be computed via linear programming.

Maximal lotteries (ML) exhibit a number of interesting properties.
First, they require little structure to be placed on  voters' preferences since $M$ is computed solely using pair-wise comparisons.\footnote{
In particular, it does not require asymmetry, completeness or even transitivity in the individual voters' preferences. }
Second, it is Condorcet-consistent, in that alternatives in the support of the maximal lottery are the Condorcet winners.  Furthermore, ML is population consistent and guarantees composition consistency (a stronger form of clone-proofness)~\citep{brandl2016consistent}. 

\begin{algorithm}[t]
\DontPrintSemicolon 
\KwIn{An initial set $A_0$ of alternatives}
\KwOut{A ranking of winner sets, $W$, and their corresponding probability distributions}
Let $W \gets ()$ be an empty sequence \;
\For{round $t \in \{0, 1, 2, \cdots, m-1\}$}{
    \If{$|A_t| = 0$}{
        {\bf break} \;
    }
    $x_t \gets \textsc{MaximalLotteries}(A_t)$ \;
    $w_t \gets \{ a \in A_t: x_t(a) > 0 \}$ \;
    Append $(x_t, w_t)$ to the end of $W$ \;
    $A_{t+1} = A_t - w_t$ \; \label{alg:reduce-alts}
    Remove the winners $w_t$ from all the votes. \; \label{alg:mod-votes}
}
\Return{$W$} \;
\caption{Iterative Maximal Lotteries (IML)}
\label{alg:iterative-ml-mainpaper}
\end{algorithm}

We propose Iterative Maximal Lotteries (IML), depicted in Algorithm~\ref{alg:iterative-ml-mainpaper}, which repeatedly removes winners and solves the margin sub-games among the remaining agents. As shown in our evaluation, IML can identify game-theoretic cycles among subsets of agents. Details about this iterative approach can be found in the appendix.

\subsection{Voting-as-Evaluation (VasE)}
\label{sec:vase_intro}
We now present our Voting-as-Evaluation framework, which, we argue, provides a principled framework for general agent evaluation.
In VasE, we define $A$ to be the set of \emph{agents}  to be evaluated. 
\emph{Voters} are different tasks or environments on which agents are being evaluated (AvT setting) or different interactions/games played (AvA setting).

We assume each voter has their own way of evaluating the performance of agents \emph{given its context}. This can be in terms of accumulated reward if considering RL agents or performance according to some particular metric. 
Figure~\ref{fig:meeple_pentathlon} provides an example where, for example, cycling is evaluated by speed, basketball is evaluated by number of baskets scored, and synchronized swimming is evaluated on a combination of technical merit and artistic expression.
Given their individual evaluation context,  voters  either generate a ranking of agents or answer pair-wise comparison queries (i.e. is agent $a$ better than agent $b$). We interpret this as the \emph{preferences} of the voters. Importantly, we do not assume that cross-task or cross game/interaction comparisons are possible, nor that all voters are using   comparable performance metrics. 
Finally, to compare agents  across multiple tasks or domains we require an aggregator that takes the ``preferences" of the different tasks and returns a collective decision. Voting rules do exactly this. Thus the evaluation problem reduces  to the problem of selecting an appropriate voting rule where  appropriateness is defined by what properties the evaluator wishes to support.
We observe that several of the consistency properties from social choice can be interpreted as desirable evaluation-framework properties. 
\begin{description}
\item[{\bf Condorcet consistency:}]  If an agent outperforms every other agent in terms of head-to-head comparisons,
Condorcet consistent voting schemes will identify this. 

\item[{\bf Population consistency:}]  When evaluating on subsets of tasks, population consistency guarantees that consensus on the subsets is conclusive.

\item[{\bf Clone consistency:}] 
The addition of near duplicates, either deliberately or arising from testing variations in architecture and parameterization, should not change the ranking.

\begin{table*}[t]
    \centering
    \begin{tabular}{l||c|c|c|c||c|c|c|c||c|}
               & App & Bor & Plu & STV & Cop & RPs & Kem & Sch & ML \\
    \hline
    Condorcet consistency  & & & & & $\checkmark$ & $\checkmark$ & $\checkmark$ & $\checkmark$ & $\checkmark$ \\
    Clone consistency & & & & & & $\checkmark$ & & $\checkmark$ & $\checkmark$ \\
    Population consistency & $\checkmark$ & $\checkmark$ & $\checkmark$ & & & & & & $\checkmark$ \\
    Complexity & $m$ & $m$ & $m$ & $m^2$ & $m^2$ & $m^3$ & $m!$ & $m^3$ & $\mbox{poly}(m^2)$\\
    \hline
    \end{tabular}
    \caption{Summary of properties of each voting method used in our evaluation.
    Note that the complexity of ML depends on that of the underlying equilibrium solver~\citep{Cohen19LPs}).
    In addition to the above, maximal lotteries also satisfies two additional properties (independence of irrelevant alternatives and agenda consistency).
    See appendix for details.}
    \label{tab:voting_method_summary}
\end{table*}
\end{description}

\begin{table*}[t!]
\centering
\begin{tabular}{|c|cccc|}
\hline
 & A & B & C & D \\
\hline
1st & \textsc{rainbow}  & \textsc{rainbow}   & \textsc{rainbow}   & \textsc{rainbow}\\
2nd & \textsc{dist-dqn} & \textsc{dist-dqn}  & \textsc{a3c}       & \textsc{dist-dqn}\\
3rd & \textsc{prio-ddqn}  & \textsc{prio-ddqn}   & \textsc{dist-dqn}  & \textsc{a3c}\\
4th & \textsc{a3c}      & \textsc{duel}      & \textsc{prio-ddqn} & \textsc{prio-ddqn}\\
5th & \textsc{duel}     & \textsc{a3c}       & \textsc{duel}      & \textsc{duel}\\
6th & \textsc{ddqn}     & \textsc{ddqn}      & \textsc{ddqn}      & \textsc{ddqn}\\
7th & \textsc{noisy-dqn}  & \textsc{noisy-dqn}   & \textsc{noisy-dqn}   & \textsc{noisy-dqn}\\
8th & \textsc{dqn}      & \textsc{dqn}       & \textsc{dqn} & \textsc{dqn}\\
\hline
\end{tabular}
\caption{Rankings found by VasE using A) Approval($k=3)$, Kemeny-Young, IML, Ranked Pairs, and Schulze, B) Borda and Copeland, C) Plurality, and D) STV. \label{tab:ale_rankings}}
\end{table*}

Since much of social choice theory is axiomatic in nature, these properties (and others) are clearly defined, supporting interpretability~\citep{peters2020explainable,procaccia2019axioms}.  
Furthermore,  its is known when certain properties are incompatible. 
This forces the evaluator to clearly articulate the trade-offs made when selecting one voting scheme over another, supporting transparency in the evaluation process. 
Consider the pentathlon of  Fig.~\ref{fig:meeple_pentathlon}, with corresponding $N(x,y)$ and $M(x,y)$ (Fig.~\ref{fig:meeple_pentathlon_N_M_mats}). If the evaluator places great importance on Condorcet consistency, they must select a rule guaranteed to return Agent $C$ as the best agent (the strong Condorcet winner). 
If, instead, the evaluator prefers population consistency, they should select a scoring rule like plurality or Borda, both of which return Agents A and C as (tied) top agents. 

\section{Empirical Evaluation}
\label{sec:eval}

To evaluate VasE, we present results of several applications, which together illustrate desirable features of our methodology. Due to space limitations, full results (include associated scores for every ranking) are deferred to the Appendix.
We use nine different voting schemes: four scoring methods (approval, Borda, plurality, Single Transferable Vote (STV))~\citep{Brandt16Handbook}, four deterministic Condorcet methods (Copeland~\citep{Copeland50}, ranked pairs~\citep{Tideman87}, Kemeny-Young~\citep{Kemeny59,young1978consistent}, Schulze~\citep{Schulze11}) and one probabilistic (Condorcet) method: maximal lotteries~\citep{Fishburn84,brandl2016consistent}.
The properties guaranteed are summarized in Table~\ref{tab:voting_method_summary}.

\subsection{General Reinforcement Learning Agents}
\label{sec:ale}

We start with an application to the Arcade Learning Environment (ALE), using the mean performance of eight agents across 54 Atari games from~\citep[Table 5]{Hessel17Rainbow}. For VasE, each vote (preference ranking) is determined by decreasing order of score, for a total of 54 complete preference rankings, one for each game. The aggregate rankings (total orders over alternatives) for all nine voting methods are shown in Table~\ref{tab:ale_rankings}.
Note that all nine methods top-rank Rainbow, which is also the strong Condorcet winner.

{\it Observation 1: ALE votes are unique.} To our surprise, despite most agents being variants of DQN, every ordering of agents is unique across 54 tasks: there are no repeated votes. The histogram of unique votes is an easy way to quantify the diversity of the task set (with respect to a set of agents) with an $O(n^2)$ enumeration over all pairs of votes.

{\it Observation 2: Only A3C changes relative ranks across the voting methods}.
This is evident from the VasE comparison, but is not be immediately obvious from inspection of~\citep[Table 5]{Hessel17Rainbow} nor from the aggregate Elo ratings. 
This could be due to A3C being the only policy gradient (a fundamentally different category of RL algorithm than DQN and its variants) agent, but suggests further investigation is warranted.
Also, A3C has many first-place wins and so is ranked second by plurality and third by STV due to their placing large importance on first-place wins.

{\it Observation 3: The alternate evaluation protocols return identical rankings.}
ALE is deterministic and so generating unique episodes for comparison requires explicit variation. Hessel et al consider two forms of this: initial periods of forced no-op actions, or taking initial actions taken from human players.
Running VasE on the alternative evaluation protocol in~\citep[Table 6]{Hessel17Rainbow} results in identical rankings across all nine voting methods. 
Thus, the different protocols had insignificant impact on the final assessments.

{\it Observation 4: Nash averaging requires score normalization.}
Recall that Nash averaging~\citep{Balduzzi18ReEval} finds an equilibrium in the Agent-versus-Task matrix induced by~\citep[Table 5]{Hessel17Rainbow}.
If we were to run Nash averaging on raw scores, the task player would place full weight on one task ({\tt skiing}) in its equilibrium because the rewards are large and negative ($<-10000$) and outside the bounds of the other games, across all agents.
To avoid this pathology, our Nash Averaging implementation
first normalizes scores into the range of $[0,1]$ as in previous work~\citep{Balduzzi18ReEval}.
In contrast, VasE compares agents' raw scores only within the same task, without requiring normalization between different tasks.

{\it Observation 5: Nash averaging has adversarial task selection in ALE.}
With normalization, Nash averaging finds a four-way tie for first place \{\textsc{a3c}, \textsc{dist-dqn}, \textsc{rainbow}, \textsc{duel}\}, with a task distribution of $(\texttt{assault}\text{:}~0.36,~\allowbreak \texttt{boxing}\text{:}~0.29,~\allowbreak \texttt{breakout}\text{:}~0.26,~\allowbreak \texttt{venture}\text{:}~0.09)$. 
The consequence of the approach is clear: first, Nash equilibrium guarantees that the top agents in the support have the same value ($0.3891$ in this case) making them indistinguishable under this metric.
Second, due to adversarial task selection, only four (7.4\% of) tasks are used to rank agents. 

{\it Observation 6: Nash averaging can be sensitive to the set of agents.}
Finally, we augment the data by adding scores of a random agent and human performance reported in~\citep[Section H.4]{Badia20Agent57}, resulting in 10 agents over 54 games. In VasE, the relative ranks of the original 8 agents remain almost identical to their counterparts in Table~\ref{tab:ale_rankings}. Rainbow remains a strong Condorcet winner, random is consistently bottom-ranked, and the position of human differs by voting method. 
Nash averaging now finds a different four-way tie: \{\textsc{a3c}, \textsc{dist-dqn}, \textsc{rainbow}, \textsc{human}\} (all with value $0.4253$) and a significantly different task distribution: $(\texttt{asterix}\text{:}~0.032,~\allowbreak \texttt{breakout}\text{:}~0.39,~\allowbreak \texttt{gopher}\text{:}~0.17,~\allowbreak \texttt{montezuma\_revenge}\text{:}~0.40)$.
Note that again Nash averaging aggressively filters 93.6\% of the data and shares only one task in its support with the 8-agent dataset.

In the appendix, we show results on $\alpha$-Rank in the ALE, compute the Elo scores for agents (Rainbow is the top-ranked agent by Elo) and apply VasE on general adaptive RL agents~\citep{AdA23}.

\subsection{Ranking Large Language Models as Agents}
\label{sec:eval_of_llms}

In this section, we apply VasE to the AgentBench benchmark evaluation data~\citep{liu2023agentbench}, that assesses the performance of 25 large language models (LLMs), including both APIs and open-source models, as agents across 8 different environments (tasks). The tasks involve operating system and database use, using a knowledge graph to answer questions, playing a digital card game, solving lateral thinking puzzles, house-holding, and browsing and shopping on the web. As in the ALE, VasE builds a preference profile from the raw task scores from \citep[Table 3]{liu2023agentbench}, without any form of score normalization. In contrast, AgentBench defines an overall score by weighting each agent's task score by the reciprocal of the task's mean score.

VasE rankings are mostly consistent with the AgentBench overall ranking: {\tt gpt-4} is the strong Condorcet winner, which is also first place in all the VasE rankings.

{\it Observation 7: Iterative Maximal Lotteries (IML) can identify qualitatively different classes of agents.} 
Table~\ref{tab:iml-agentbench} shows the top 8 ranks and scores by iterative maximal lotteries on AgentBench.
The scoring system for IML is interpretable as ``levels'', where each level corresponds to a set of agents that were selected as winners in the margin subgame at each iteration, and an agent's score is equal to $l - 1 + w_a$, where $l$ is the level they achieved and $w_{t,a}$ is the probability they were assigned in the equilibrium of the game. For example, in the first iteration (out of 17 total) {\tt gpt-4} is the only winner (selected with probability 1),
so it achieves a score of $17 - 1 + 1 = 17$.

\begin{table}[t]
\begin{center}
\begin{tabular}{|c|ll|}
\hline
Rank & Agent & Score\\
\hline
1 & {\tt gpt-4} & 17.00\\
2 & {\tt claude} & 16.00\\
3 & {\tt gpt-3.5-turbo} & 15.00\\
4 & {\tt text-davinci-003} & 14.00\\
5 & {\tt claude-instant} & 13.00\\
6 & {\tt text-davinci-002} & 12.00\\
7 & {\tt chatglm} & 10.56\\
8 & {\tt t-bison-001} & 10.43\\
\hline
\end{tabular}
\end{center}
\caption{Top 8 ranked agents using iterative maximal lotteries. Integer score differences between ranks represent Condorcet winners in the margin subgame that was solved. Each agent in the top six was chosen with probability 1, whereas the bottom two were part of Nash equilibrium distribution assigning probability 0.56 to {\tt chatglm} and 0.43 to {\tt t-bison-001}. \label{tab:iml-agentbench}}
\end{table}

This scoring system conveys qualitative information about results of each iteration. The top six agents have integer scores: they were Condorcet winners of their respective margin subgames. In other words, the higher-level agents dominate the lower, 
in terms of performing better on a greater number of tasks.
Looking only at the raw data in~\citep[Table 3]{liu2023agentbench}, {\tt text-davinci-002} and {\tt t-bison-001} seem similar: their overall scores differ only by 0.07 and the scores across tasks are comparable. 
However, Table~\ref{tab:iml-agentbench} above shows these agents perform at different levels. 
This difference is confirmed by looking at the margin subgame between the bottom three agents: {\tt text-davinci-002} is a Condorcet winner of the margin subgame which excludes the top six agents.
Though neither Elo nor the aggregate score make it obvious, {\tt text-davinci-002} is a stronger performer on these tasks than {\tt text-bison-00l} and {\tt chatglm}.

In the appendix, we show another game-theoretic cycle found by IML. It also provides example visualizations for greater interpretability on the Chatbot Arena~\citep{zheng2023judging} dataset.

{\it Observation 8: Nash averaging suffers from adversarial task selection in AgentBench.} We run Nash averaging on AgentBench in a similar way as on ALE, by normalizing the scores into $[0, 1]$. As in ALE, Nash averaging finds a low-support equilibrium, with a task distribution of 
$(\texttt{HouseHolding}\text{:}~0.135,\allowbreak \texttt{WebShopping}\text{:}~0.865)$.
Again, a large portion ($\frac{6}{8} = 75\%$) of the tasks are not used in ranking the agents. The task player places most weight on the only task that {\tt gpt-4} does not place 1st. Under this distribution, there is a two-way tie for top rank between {\tt gpt-4} and {\tt gpt-3.5-turbo} with a Nash average value of 0.889. 

See the appendix for further results using VasE for LLM evaluation. There, we apply VasE to a prompt engineering for text summarization task and to the Holistic Evaluation of Language Models (HELM)~\citep{liang2022holistic} data.

\subsection{Ranking Human Players in Diplomacy}
\label{sec:eval-ava}

In the AvA setting evaluation systems should be predictive of individual games/matchups. In this subsection, we measure how well voting methods can predict the rank of human players in the game of Diplomacy, a 7-player mixed-motive game. Diplomacy involves negotiation, as well as the repeated formation and betrayal of alliances, in an effort to be the first to control 18 supply centers. 
VasE can be applied here by interpreting each 7-agent evaluation task (one game played between 7 players) as a ranking over those $7$ agents (\ie a single vote). 
The simplest case is to rank agents by their final outcomes (the terminal payoffs) over each game. 
We use a dataset of anonymized human games played on webDiplomacy~\citep{webdip} between 2008 and 2019.
This dataset contains $n = 31049$ seven-player games played by $m = 52958$ players, and 
has also been used to train and evaluate Diplomacy RL agents~\citep{Paquette19,Anthony20Dip,Gray20Humanlevel,Cicero}.
Note that a payoff tensor of the meta-game would require $7 \cdot 52958^7 \approx 10^{32}$ entries, which is infeasible to solve using Nash averaging, $\alpha$-rank, or correlated equilibrium based ranking. Even the much smaller game solved by maximal lotteries would have $52958^2 \approx 2.8$ million entries.

We run VasE on this dataset and compare it to Elo (a slight variant of which is used by webDiplomacy~\citep{AnthonyGhostRatings}).
The number of supply centers at the end of the game determines the ordering of players for that game.
In VasE, this leads to a number of voters equal to the number of games in the dataset.
In Elo, each end-of-game rank contributes ${7 \choose 2} = 21$ head-to-head match outcomes between players $(i,j)$, with a win for $i$ over $j$ represented as player $i$ achieving more supply centers than player $j$.

\begin{figure}[t]
\includegraphics[width=1.0\linewidth]{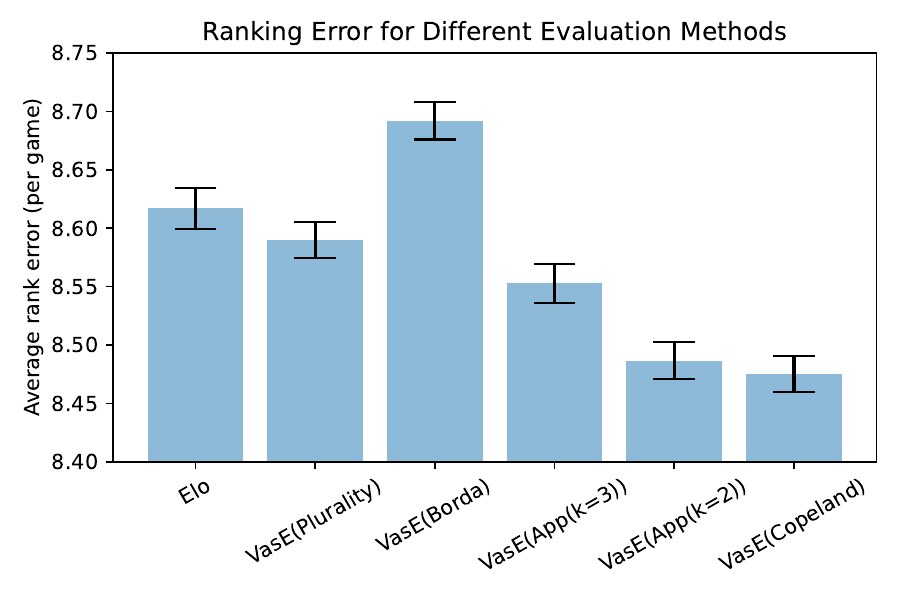} 
\caption{Average ranking error across 50 splits. Error bars represent 95\% confidence intervals.}
\label{fig:vase_webdip}
\end{figure}

Given a set of games $G$, an evaluation system returns a ranking, $R(G)$, over the agents. We compare this evaluation ranking to the actual agent-ranking for some specific game instance, which we denote by $g$. This allows us to infer how accurate our methods are in capturing the true agent-ranking.  To this end, we use the Kendall-tau metric, $K_d(R(G),g)$, which counts the number of pairwise disagreements between two rankings, where the smaller the distance the more similar two rankings are~\citep{kendall1938new}. 
We randomly generate 50 train/test splits, $(\mathcal{D}_R, \mathcal{D}_T)$, each with $|\mathcal{D}_R| = 28049$ and $|\mathcal{D}_T| = 3000$.
Each evaluation computes a player rank $R(\mathcal{D}_R)$ using $\mathcal{D}_R$, and incurs an average test error per game
$\textsc{ErrorPerGame}(\mathcal{D}_T, R(\mathcal{D}_R))= \frac{1}{|\mathcal{D}_T|} 
\sum_{g \in \mathcal{D}_T} K_d(R(\mathcal{D}_R), g)$. 
Results are shown in Fig.~\ref{fig:vase_webdip}.

{\it Observation 9: Some voting methods generalize better than Elo on outcome prediction in Diplomacy.}
Of the voting methods, we observe that approval voting with $k=3$ and $k=2$ and Copeland have significantly lower error than Elo. Plurality has slightly lower error (not significant), and Borda significantly higher. Copeland (the only Condorcet method) performed best: this is interesting, as Copeland bases its score directly on pairwise comparisons in the same way as Elo does in this 7-player game. Unlike Elo, Copeland scores are free from satisfying head-to-head win prediction semantics. 
This freedom might be why Copeland provides better generality than Elo in this complex 7-player negotiation task.

\section{Conclusion and Future Work}
\label{sec:conc_future_work}

In a recent paper  ``Rethink reporting of evaluation results in AI’’, the authors proposed  guidelines for robust evaluation practices, the first of which included the statement ``\emph{Aggregation decisions should be clearly explained, and analyses conducted to explore system performance should be described}”~\citep{burnell2023rethink}. Our framework, Voting-as-Evaluation (VasE), does precisely this. By leveraging key insights from social choice theory and adapting them for the general-agent evaluation problem, we provide explainable general-purpose evaluation schemes~\citep{peters2020explainable,procaccia2019axioms}.  Through extensive experiments across multiple domains 
we illustrated the flexibility of our approach, and showed how VasE addresses challenges that alternative evaluation schemes (Elo and Nash Averaging) encounter. We observed that: data from evaluation problems is rich, as 
 illustrated by the diversity of agent rankings across tasks; that 
 game-theoretic approaches to evaluation place constraints that may lead to counter-intuitive findings, such as often using only a small subset of the evaluation data; that VasE provides a principled way of doing general evaluation using diverse, possibly non-comparable metrics across many scenarios; and that VasE naturally extends to the  AvA setting, and can generalize better when predicting human per-game player rankings in a seven-player game than a standard rating system (Elo).
 
 The flexibility of VasE is one of its strengths. The evaluator has a choice of voting rules at their disposal, each with well defined properties and tradeoffs. Thus, the evaluator, by selecting one rule over another, is clearly communicating the properties they deem to be important. While we intentionally leave this choice to the evaluator, in our opinion maximal lotteries are particularly well suited for evaluation problems. Maximal lotteries exhibits many of the properties we deem to be important, is computationally tractable, requires only pairwise comparisons, and the probabilistic nature of the scheme allows for nuanced interpretations of results that are critical when comparing agents and systems. 
 
 We identify directions for future work. Recent research on robust voting may provide methods to capture and represent uncertainty in the evaluation process~\citep{boehmer2022quantitative}, particularly when agent-performance is stochastic or sensitive to parameterizations. VasE explicitly uses only rankings or pairwise comparisons, ignoring cardinal information such as scores or rewards which may contain useful signals. Research on distortion (\eg~\citep{anshelevich2021distortion}) may provide insights as to how to enrich VasE findings. More broadly, we believe the general-agent evaluation problem serves as a bridge between different research communities and facilitates sharing sharing problems, data, and methodologies, for the betterment of all.
 
 \subsection{Acknowledgments}
 
 We would like to thank Manfred Diaz for continued discussions on evaluation methodologies, Felix Brandt for his feedback on earlier versions of this paper, and Ratip Emin Berker for feedback and suggested improvements. 

\bibliography{main}

\appendix

\onecolumn

\section{Nash Averaging}
\label{app:nash_averaging}

\begin{lemma} [Symmetric MENE] \label{lemma:symmetric_mene}
    In a two-player, symmetric, zero-sum game, the maximum entropy Nash equilibrium (MENE) results in unique and equal mixed strategies for both players.
\end{lemma}
\begin{proof}
    Let an NE be the tuple $(x_1, x_2)$ in a two-player zero-sum game. If there is a second NE, $(x'_1, x'_2)$, any mixture over $x_1$ and $x'_1$, or $x_2$ and $x'_2$ is also an NE (convexity). Furthermore, $(x_1, x'_2)$ and $(x'_1, x_2)$ are also an NEs (interchangeable). Therefore the maximum of any strictly concave function will select uniquely from the set of Nash equilibria. Maximum entropy therefore performs equilibrium selection uniquely. If the game is symmetric, $\{ x_1, x'_1, ... \} = \{ x_2, x'_2, ... \}$, both players have the same components of the NE. Therefore, the MENE of symmetric, two-player, zero-sum games has the property that $x_1=x_2$.
\end{proof}

\begin{lemma} [Entropy of Product of Marginals] \label{lemma:marginals_entropy}
    When $x_i \geq 0$, $x_j \geq 0$, $\sum_i x_i = 1$ and $\sum_j y_j = 1$,
    \begin{align}
        - \sum_{ij} x_i y_j \ln \left( x_i y_j \right) = - \sum_{i} x_i \ln \left( x_j \right) - \sum_{j} y_j \ln \left( y_j \right).
    \end{align}
\end{lemma}

The Nash averaging work \citep{Balduzzi18ReEval} formulates the AvT problem (Section D) as a symmetric, two-player, zero-sum \emph{double game}. A double game has two games within it, and each player, $p = \{1, 2\}$, has two ``subplayers'', $\{A, B\}$, that each control a mixed strategy, $\{x_p^A, x_p^B\}$. There is a single asymmetric score matrix $S$. This is an unusual formulation, and was presented this way to make the conditions compute Multidimensional Elo \citep{Balduzzi18ReEval} clearer. However, when computing the Nash average, there is an equivalent \emph{single game} formulation that is simpler to reason about.

\begin{theorem}[Nash Average AvT Equivalence]
    \begin{align}
        \arg\limits_{x_1^A, x_2^A} \min_{x_1^A, x_1^B} \max_{x_2^A,x_2^B}
        \begin{bmatrix}
            x_1^A \\
            x_1^B
        \end{bmatrix}^T
        \begin{bmatrix}
            0 & S \\
            -S^T & 0
        \end{bmatrix}
        \begin{bmatrix}
            x_2^A \\
            x_2^B
        \end{bmatrix} %
    \end{align}
    $ = \arg \min_{x_1^A} \max_{x_2^B} (x_1^A)^T S x_2^B$, 
    where $\mathbf{1}^T x_p^A = 1$, $\mathbf{1}^T x_p^B = 1$, $x_p^A \geq 0$ and $x_p^B \geq 0$.
\end{theorem}
\begin{proof}
    First, note that the optimization can be decomposed into the sum of two separate optimizations.
    \begin{align}
        &\phantom{=} \min_{x_1^A, x_1^B} \max_{x_2^A, x_2^B}
        \begin{bmatrix}
            x_1^A \\
            x_1^B
        \end{bmatrix}^T
        \begin{bmatrix}
            0 & S \\
            -S^T & 0
        \end{bmatrix}
        \begin{bmatrix}
            x_2^A \\
            x_2^B
        \end{bmatrix} \\
        &= \left( \min_{x_1^A} \max_{x_2^B} (x_1^A)^T S x_2^B \right) + \left( \min_{x_1^B} \max_{x_2^A} (x_1^B)^T (-S^T) x_2^A \right) \\
        &= \left( \min_{x_1^A} \max_{x_2^B} (x_1^A)^T S x_2^B \right) - \left( \min_{x_2^A} \max_{x_1^B} (x_2^A)^T S x_1^B \right) = 0
    \end{align}
    As expected, this evaluates to a game value of zero. By Lemma~\ref{lemma:symmetric_mene}, we know that $x^A = x_1^A = x_2^A$ and $x^B = x_1^B = x_2^B$ when using the maximum entropy criterion. Therefore, in order to find the MENE of the double game, by Lemma~\ref{lemma:marginals_entropy} we only need to find the MENE of the two-player zero-sum asymmetric single game, where player 1 has payoffs $S$.
\end{proof}

\section{Elo Rating Analysis of Example in Figure~\ref{fig:meeple_pentathlon}}
\label{app:elo_meeple}

Recall that, when applying Elo to the example in Figure~\ref{fig:meeple_pentathlon}, every time $x$ is preferred to $y$ counts as a win for $x$ and loss for $y$. This leads to the win rates depicted in Figure~\ref{fig:elo-meeple-winrates}, or the following matrix:
\[
\begin{bmatrix}
        & 0.8 & 0.4 \\
    0.2 &     & 0.4 \\
    0.6 & 0.6 &     \\
\end{bmatrix}
\]

Suppose we start with initial parameters $r^0_A = r^0_B = r^0_C$. Each iteration $t$, Elo applies an update rule of the form:
\[
r^{t+1}_x = r^t_x + \Delta_{r^t_x}
\]
where $\Delta_{r^t_x}$ is the magnitude of the gradient of $r^t_x$, defined below.

Suppose agents $x$ and $y$ play a game $g$ and $x$ wins (achieving a score of 1) and y loses (achieving a score of 0). The contribution of the change in ratings from just this single game is~\citep{Elo78}:
\[
\Delta_{Elo}(g, r_x, r_y, x, y) = K \cdot (S(g, x, y) - E(r_x, r_y)),
\]
where $K$ is a constant (usually 16 or 32), $S(g, x, y) = 1$ is the score achieved by agent $x$ against $y$ in game $g$, and $E(r_x,r_y)$ is the expected score according to the logistic model:
\[
E(r_x, r_y) = \frac{1}{1 + e^{(r_y - r_x)/400}}.
\]
Note that $E(r_x, r_y) = 1 - E(r_y, r_x)$.

If there were only one game played, between agent $x$ and agent $y$, the rating $r_x$ would increase by $\Delta(g, r_x, x, y)$ and $r_y$ would decrease by $\Delta(g, r_x, x, y)$.
Given a batch of head-to-head games $\cG$, let $\cG(x)$ be the games involving $x$. The gradient component of $r_x$ is:
\[
\Delta_{r^t_x} = \sum_{g \in \cG(x)} \Delta(g, r_x, r_y, x, y) - \Delta(g, r_y, r_x, y, x).
\]
For the game depicted in Figure~\ref{fig:elo-meeple-winrates}, this leads to:
\begin{align*}
\Delta_{r^t_A} & = &   & 4 K (1 - E(r_A, r_B))\\
              &   & -  & 1 K (1 - E(r_B, r_A))\\
              &   & +  & 2 K (1 - E(r_A, r_C))\\
              &   & -  & 3 K (1 - E(r_C, r_A))\\
              & = & 2K & - 4K \cdot E(r_A, r_B) + K \cdot E(r_B, r_A)\\
              &   &    & - 2K \cdot E(r_A, r_C) +3K \cdot E(r_C, r_A)\\
              & = & 2K & - 4K \cdot E(r_A, r_B) + K (1 - E(r_A, r_B))\\
              &   &    & - 2K \cdot E(r_A, r_C) +3K (1 - E(r_A, r_C))\\
              & = & 6K & - 5K \cdot E(r_A, r_B) -5K \cdot E(r_A, r_C), \\
\end{align*}
and
\begin{align*}
\Delta_{r^t_C} & = &   & 3 K (1 - E(r_C, r_A))\\
              &   & -  & 2 K (1 - E(r_A, r_C))\\
              &   & +  & 3 K (1 - E(r_C, r_B))\\
              &   & -  & 2 K (1 - E(r_C, r_B))\\
              & = & 2K & - 3K \cdot E(r_C, r_A) + 2K \cdot E(r_A, r_C)\\
              &   &    & - 3K \cdot E(r_C, r_B) + 2K \cdot E(r_B, r_C)\\
              & = & 2K & - 3K \cdot E(r_C, r_A) + 2K (1 - E(r_C, r_A))\\
              &   &    & - 3K \cdot E(r_C, r_B) + 2K (1 - E(r_C, r_B))\\
              & = & 6K & - 5K \cdot E(r_C, r_A) -5K \cdot E(r_C, r_B)).\\
\end{align*}

It is easy to see that at iteration $t = 0$, $\Delta_{r^t_A} = \Delta_{r^t_C}$ since all the ratings are equal and $E(r_A, r_B) = E(r_A, r_C) = E(r_C, r_A) = E(r_C, r_B) = \frac{1}{2}$. So $r_A$ and $r_C$ will be incremented by the same amount, and the rating of $r_B$ will decrease by some amount.
At iteration $t > 0$, it will still be the case that $E(r_A, r_C) = E(r_C, r_A) = \frac{1}{2}$ because $r^1_A = r^1_C$. So the only difference in the terms of the gradients is that one has $E(r_A, r_B)$ and the other has $E(r_C, r_B)$. But, it is also the case that these are equal: $E(r_A, r_B) = E(r_C, r_B)$, since $r_C = r_A \Rightarrow r_C - r_B = r_A - r_B$.
So the gradients $\Delta_{r^t_A} = \Delta_{r^t_C}$ for all $t \ge 0$.

\section{Social Choice and Voting Methods}
\label{app:voting_methods}

In this section we provide technical details pertaining to the key concepts we use from social choice theory.  We also introduce the different voting rules used in the experiments, including our iterative version of maximum lotteries, which may be of independent interest.

\subsection{Properties}\label{app:cons_props}
We provide formal definitions of key concepts that were introduced in Section~\ref{sec:bg-soccho}.
Recall that a \defword{voting scheme} is defined by $\langle A, V, f\rangle$ where $A$ is the set of alternatives, $V$ is the voting population, and $f$ is a voting rule. Furthermore, voters' preferences are captured in a preference profile, $[\succeq ]$, a vector specifying the preferences of each voter.

Given a specific preference profile, it is often useful to know how many voters prefer an alternative $x$ over some other alternative $y$. Define
$$N(x,y) = |\{ v\in V|x\succ_v y\}| $$
where $N(x,x)=0$. Furthermore, define the \defword{vote margin} between alternatives $x$ and $y$ as
$$ \delta(x,y)=N(x,y)-N(y,x).$$

\begin{definition}[Condorcet Winner~\citep{deCondorcet1785}]
Given a specific preference profile, $[\succeq]$,  alternative $a\in A$ is a  (weak) \defword{Condorcet winner}  if $\forall a'\in A$, $\delta(a,a')\geq 0$. If $\forall a'\in A, \delta(a,a')>0$ then $a$ is a \defword{strong} Condorcet winner.
\end{definition}

We observe that a Condorcet winner might not exist. Consider the following scenario with 3 voters with the following preferences over alternatives $A$: $\{a\succ_1 b\succ_1 c, b\succ_2 c\succ_2 a, c\succ_3 a \succ_3 b\}$. 
No single alternative is even a weak Condorcet winner due to the cycle in the aggregate preferences. 
\begin{definition}[Condorcet Consistency]
A voting scheme is \defword{Condorcet consistent} if it identifies Condorcet winners when they exist.\footnote{By \emph{identifies} we mean that either the set of alternatives returned by a social choice function contains Condorcet winners or the social welfare function, top-ranks Condorcet winners.}

\end{definition}

\begin{definition}[Population Consistency~\citep{smith1973aggregation}]
Given two disjoint sets of voters, $V, V'$, $V\cap V' = \emptyset$ with preference profiles $[\succeq]$ and $[\succeq']$, a voting scheme is \defword{population consistent} if 
\[
f([\succeq])\cap f([\succeq']) \not= \emptyset\]
then
\[
f([\succeq, \succeq'])\subseteq f([\succeq])\cap f([\succeq']).
\]
If this holds with equality, then $f$ is \defword{resolute}.
\end{definition}

Recall that the class of scoring rules are population consistent.
A scoring rule is characterized by  a score vector $\mathbf{w}=(w_1,\ldots, w_m)$ where $w_1\geq w_2\geq \ldots \geq w_m$ and $w_1>w_m$. Each voter, given its ranking $\succeq_i$, assigns score $w_1$ to its top-ranked alternative, $w_2$ to the second-ranked alternative, and so on.
An overall score for each alternative is computed by summing across all weights assigned to it  by the different agents in $V$, and the set of alternatives with the highest score is returned as the outcome of the social choice function, $f_\mathbf{w}(\cdot)$ (or a ranking where alternatives are ordered by their scores is returned).  That is, given $\mathbf{w}$ and $[\succeq]$, where $w^j(a)$ is the score voter $j$ assigns to alternative $a$,
$$SCF_\mathbf{w}([\succeq])=\arg\max_{a\in \mathcal{A}} \sum_{j\in V} w^j(a).$$ 
Observe that social welfare functions are also well-defined via scoring vectors, where in the aggregated ranking alternatives are ranked according to their weighted sums.
Many familiar voting schemes belong to this class of voting rules, including plurality, Borda, and $k$-approval. 
For example, the Borda rule is defined by scoring vector $\mathbf{w}=(m-1,m-2,\ldots,1,0)$, Plurality corresponds to $\mathbf{w}=(1,0,0,\ldots,0)$,  while $k$-approval  is captured by $\mathbf{w}=(1_1,\ldots, 1_k,0,\ldots,0)$, where the first $k$ alternatives are deemed acceptable to the agent.

The next key property is \defword{clone consistency}. Before introducing this property we must first introduce some terminology.
Let $A$ and $B$ be sets of alternative such that $B\subseteq A$. Given $[\succeq]$ $B$ is a \emph{component} of $A$ if for all  $a\in A\setminus B$ and $b,b'\in B$, $a\succeq b$ if and only if $a\succeq b'$. That is, $b$ and $b'$ are \emph{clones} of each other since they have the same relationship to all alternatives not in $B$. 

A voting rule is clone consistent if two conditions hold~\citep{Tideman87}:
\begin{enumerate}
    \item 
    If the voting rule returns a clone then it must also be the case that when any member of the clone set is removed and the voting rule is applied, then it  still returns a clone.
   
    \item If the voting rule returns an alternative $a$ that is not in the clone set then it must be the case that the voting rule also returns $a$ when any member of the clone set is removed from consideration.
   
\end{enumerate}

\begin{definition}[Clone consistency~\citep{Tideman87}]
Let $\langle A, V, f\rangle$ be a voting scheme, and let $[\succeq]$ be the preferences of the voters over $A$.  Let $B$ be a set of clones under $[\succeq]$. Furthermore, we use notation $[\succeq]_{|_{{C}\subset A}}$ to mean the preferences of voters restricted to alternatives in $C\subset A$.  Voting rule $f$ is clone consistent if:
\begin{enumerate}
\item For any $x\in B$
    $B\cap f([\succeq])\not = \emptyset \Leftrightarrow B\setminus \{x\} \cap f([\succeq]_{|_{A\setminus \{x\}}})\not = \emptyset$
    
\item For any $x\in B$ and $y\in A$ such that $x$ is a clone but $x$ is not a clone of $y$, 
$y\in f([\succeq])\Leftrightarrow y\in f([\succeq]_{|_{A\setminus\{x\}}})$
\end{enumerate}
\end{definition}

The final property of interest is \defword{agenda consistency}, a strong property that is only supported by probabilistic social choice functions. First note that an agenda, $B\subseteq A$ is simply a subset of alternatives over which voters have preferences. As before let $[\succeq]_{|_{B}}$ be voters preferences restricted to some agenda $B$, and recall that voting rule $f([\succeq])$ returns a distribution $p\in \Delta(A)$.
\begin{definition}[Agenda Consistency]\label{def:agenda}
Given (probabilistic) voting scheme $\langle A, V, f\rangle$ with $f([\succeq])=p$, let $B, C\subseteq A$ be two agendas such that $\mathrm{supp}(p)\in B\cap C$. Then voting rule $f$ is agenda consistent if
$$p\in f([\succeq]_{|_B})\cap f([\succeq]_{|_C}) \Leftrightarrow p\in f([\succeq_{|_{B\cup C}}]).$$
\end{definition}
The careful reader will observe that this definition is the combination of Sen's $\gamma$-expansion and $\alpha$-contraction properties~\citep{sen1971choice,sen1977social}.

\subsection{Voting Methods used in the Applications}

\begin{table*}[t!]
    \centering
    \begin{tabular}{l||c|c|c|c||c|c|c|c||c|}
               & App & Bor & Plu & STV & Cop & RPs & Kem & Sch & ML \\
    \hline
    Condorcet consistency  & & & & & $\checkmark$ & $\checkmark$ & $\checkmark$ & $\checkmark$ & $\checkmark$ \\
    Clone consistency & & & & & & $\checkmark$ & & $\checkmark$ & $\checkmark$ \\
    Population consistency & $\checkmark$ & $\checkmark$ & $\checkmark$ & & & & & & $\checkmark$ \\
    Agenda consistency & & & & & & & & & $\checkmark$ \\
    Ind. Irrelevant Alt (IIA) & $\checkmark$ & & & & & & & & $\checkmark$ \\
    Complexity & $m$ & $m$ & $m$ & $m^2$ & $m^2$ & $m^3$ & $m!$ & $m^3$ & $\mbox{poly}(m^2)$\\
    \hline
    \end{tabular}
    \caption{Summary of properties of each voting method used in our evaluation.}
    \label{tab:voting_method_summary_full}
\end{table*}

In this subsection, we describe each of the voting methods used in the applications of VasE in the paper. See Table~\ref{tab:voting_method_summary_full} for a summary of the properties each method.

We use the pentathlon example depicted in Figure~\ref{fig:meeple_pentathlon} as a simple recurring example. Recall the preference profile $[\succeq] = $

\begin{center}
\begin{tabular}{c|c}
{\bf Weight} & {\bf Vote} \\
\hline
$1$ &  $A \succ B \succ C$ \\
$1$ &  $A \succ C \succ B$ \\
$2$ &  $C \succ A \succ B$ \\
$1$ &  $B \succ C \succ A$ \\
\hline
\end{tabular}
\end{center}
which leads to preference and margin matrices
\[
N = \left( \begin{array}{ccc}
0 & 4 & 2 \\
1 & 0 & 2 \\
3 & 3 & 0 \end{array} \right),
~~~
M = N - N^T = \left( \begin{array}{ccc}
0 & 3 & -1 \\
-3 & 0 & -1 \\
1 & 1 & 0 \end{array} \right).
\]
Immediately we see that $C$ is a Condorcet winner because all of the values in the last row of $M$ are positive except the diagonal (which is always zero). So all of the Condorcet methods will top-rank C.

\subsubsection{Approval Voting}

In approval voting, the top $k$ alternatives in each vote are given 1 points, and the score is determined by total points earned. When $k = 2$, $A$ has a score of $1 + 1 + 2 = 4$, $C$ has a score of $1 + 2 + 1 = 4$, and $B$ has a score of $1 + 1 = 2$. Since approval voting is deterministic, a tie-breaker rule would be needed to determine a winner and totally-ordered social welfare function in this case.

Approval voting is a scoring method and hence population consistent. The score is computable by summing over the $k$ top entries, so the complexity is linear in $m$.

\subsubsection{Plurality ($1$-Approval)}

Plurality defines an alternative's score as the number of votes in which it was chosen as the top rank. In the example, this leads to $A$ with a score of 2, $B$ with a score of 2, and $B$ with a score of 1. Since plurality is a deterministic rule, a tie-breaker would have to be used to determine which of $A$ or $C$ is top-ranked in the social welfare function.

\subsubsection{Borda}

Borda will assign points to each alternative depending on their place in the vote $(m-1, m-2, m-3, \cdots, 0)$, and an alternative's score is their total points over all the votes. 
$A$ scores $2 + 2 + 2\cdot1 = 6$, $B$ scores $1+2 = 3$, and $C$ scores $1 + 2\cdot2 + 1 = 6$.

\subsubsection{Single Transferable Vote (STV)}

STV is a multi-winner method similar to plurality that attempts to approach proportional representation by re-assigning the interpretation of surplus votes. 

Suppose the number of winners is determined to be $k$.
The election proceeds in rounds. In each stage, there are a number of total votes ($n$) and number of remaining candidates ($m$).
On each round, either a winner is added to the set of winners or a loser is added to the set of losers. The process stops if the number of winners is $k$. 

The following takes place in each round. First, a ``quota'' is defined as the amount of votes that a winner must achieve to be declared a winner.
Let $n$ be the total number of votes and $k$ be the number of winners. The Droop quota
$\lfloor \frac{n}{k + 1} + 1 \rfloor$ is used for this. If there are any alternatives that reach the quota, they are elected. 
Votes that were used to elect the winner (up to the quota) are removed, and any
surplus votes (votes that go beyond the quota) are transferred to a lower-ranking alternative, \eg the top-ranked alternative is deleted from the vote and the next alternative in each surplus vote becomes the top-ranked alternative.
If there are not enough votes to reach the quota for this round, the alternative with the lowest number of votes is eliminated, and all of the votes that top-ranked the losing alternative are transferred in a similar way.

We define a social welfare function to be the winners sorted by how early they were elected (earliest being ranked higher) followed by the losing alternatives sorted by how late they were eliminated (latest being ranked higher). 

We also define a score for STV. 
Suppose the ranking is a concatenation of list $W$ (winners) and $L$ (losers).
When there is a winner in a round, they are appended to $W$, when there is a loser in the round, they are prepended to $L$.
The winners then have a base score of $2m - i_w$, where $i_w$ is the winner's index in $W$.
The losing alternatives have a base score of $m - i_l$, where $i_l$ is the index of the loser in $L$. 
To these base values we concatenate ``.X'' where $X$ is the number of votes they received in the round they won or were eliminated.
Each scores are interpretable: if it exceeds $m$, they are a winner, and their position determines which round they won in; if the score is less than or equal to $m$ then the alternative lost and the order determines which round they lost. The score encodes the number of votes as extra information after the period.

By default, if the number of winners is not specified, we choose $k = \min(1, \lfloor\frac{m}{2} \rfloor)$.

In the running example from Figure~\ref{fig:meeple_pentathlon}, STV ranks the alternatives $C \succ A \succ B$ with scores $(6.3, 3.2, 2.1)$ for alternatives $C$, $A$, and $B$, respectively.

\subsubsection{Copeland}

Copeland's method assigns a score to alternative $a \in A$ equal to the number of other candidates beaten head-to-head the half times the number of head-to-head ties.
Define $A_{\neg a} = A - \{ a \}$. Then, $\textsc{Copeland}([\succeq], a) =$
\[ 
| \{ b \in A_{\neg a} : \delta(a,b) > 0 \} | + \frac{1}{2} | \{ b \in A_{\neg a}: \delta(a,b) = 0 \} |.
\]
From the margin matrix $M$: the Copeland score of $A$ is 1, $B$ is 0, and $C$ is 2.
So the overall ranking is $C \succ A \succ B$.

Copeland is Condorcet consistent but not clone-proof. It only has to loop over each row in the margin matrix once to compute each alternative's score, so its complexity is $O(m^2)$.

\subsubsection{Kemeny-Young}
\label{app:kemeny-young}

Let $\vec{a}$ denote a specific sequence over all the alternatives in $a \in A$.  
Let $a[i]$ denote the alternative in vote $\vec{a}$ at position $i$ (0 is the top (most preferred), and $|\vec{a}|-1$ is the bottom (least preferred). 

Let $Z(\vec{a}) = 
\{ 0, 1, \cdots, |\vec{a}| - 1\}$, and
$Z_{i,j}(\vec{a}) = \{ (i, j)~|~(i, j) \in Z(\vec{a}) \times Z(\vec{a}), j > i \}$.
Define the Kemeney value of a sequence as the sum of preference margins between every two pair of alternatives in $\vec{a}$ where one $(j)$ is ranked lower than the other ($i$) in $\vec{a}$:  
\[
\textsc{KemenyValue}(\vec{a}) = \sum_{Z_{i,j}(\vec{a})} N(a[i], a[j]).
\]

The Kemeny-Young voting method~\citep{Kemeny59,young1975social} enumerates all possible $|A|!$ permutations, choosing the one that maximizes the Kemeny value $\argmax_{\vec{a}} \textsc{KemenyValue}(\vec{a})$.

While Kemeny-Young is not a scoring method, 
we define a sensible score for each individual.
Let $Z_{i=pos(a), j} = Z_{i,j} \cap \{ (i, j) \in Z_{i,j} \mbox{ s.t. } i = pos(a) \}$, where $pos(a)$ is alternative $a$'s position in $\vec{a}$.
The score is then the sum of preference margins versus all the agents that rank lower than $a$ it in the sequence $\textsc{KemenyScore}(\vec{a}, a) = $
\[
    \left\{ \begin{array}{ll}
         0 & \mbox{if $pos(a) = |\vec{a}| - 1$};\\
        \sum_{(i, j) \in Z_{pos(a), j} } N(a, a[j]) & \mbox{otherwise}.\end{array} \right.
\]

Under these definitions, the Kemeny method attributes the following scores to the example from Figure~\ref{fig:meeple_pentathlon}:

\begin{center}
\begin{tabular}{cc}
{\bf Ranking} & {\bf Kemeny Value}\\
\hline
$A \succ B \succ C$  & $8$\\
$A \succ C \succ B$  & $9$\\
$B \succ A \succ C$  & $5$\\
$B \succ C \succ A$  & $6$\\
$C \succ A \succ B$  & $10$\\
$C \succ B \succ A$  & $7$\\
\hline
\end{tabular}
\end{center}

In this case, $C \succ A \succ B$ is chosen, and the scores are $(6, 4, 0)$ for $C$, $A$, and $B$, respectively.

\subsubsection{Schulze}

The Schulze method constructs a graph between candidates and weighs the alternative using the strength of paths between each alternative~\citep{Schulze11}.

Let $p$ represent a path between two candidates $a, b \in A$, and $pos(p, c)$ (defined similarly to Kemeny-Young, 
the position/index of alternative $c$ in $p$).
A valid path $p$ between two candidates $a, b \in A$ is a sequence where
\begin{itemize}
    \item the first alternative in the sequence (position 0) is $a$, and
    \item the last alternative in the sequence (position $m-1$) is $b$, and
    \item for every pair of alternatives $(i, j)$ next to each other in the order ($pos(p, j) = pos(p, i)+1$), the number of votes that prefer the agent at position $i$ is higher than those who prefer the agent at position $j: \delta(p[i], p[j]) > 0$.
\end{itemize}

Define $Z(p) = \{ (0, 1), (1, 2), (2, 1), \cdots, (m-2, m-1) \}$.
The strength of the path as its weakest link:
\[
\textsc{PathStrength}(p) = \min_{i,j \in Z(p)} N(p[i], p[j]).
\]

A $m$-by-$m$ matrix $P$ is constructed where each entry between two alternatives represents the strongest path between the two alternatives: $P(a, b) = \max_{p} \textsc{PathStrength}(p)$. The Schulze method then assigns a social welfare function rank:
\[
a \succ b \Leftrightarrow P(a,b) > P(b,a).
\]

Similarly to Kemeny, scores are not defined in the case of Schulze, but we define one to be the sum of alternatives that are preferred an alternative $a$ to all other alternatives $b$ lower than $a$ in the rating.
\[
\textsc{SchulzeScore}(\succ, a) = \sum_{a \succ b} N(a,b). 
\]

In the recurring example from Figure~\ref{fig:meeple_pentathlon}, the strongest paths are:

\begin{center}
\begin{tabular}{ccc}
{\bf Path End Points} & {\bf Path} & {\bf Path Strength}\\
\hline
$A \rightarrow B$ & $(A, B)$   & 4\\
$C \rightarrow A$ & $(C, A)$   & 3\\
$C \rightarrow B$ & $(C, B)$   & 3\\
\hline
\end{tabular}
\end{center}
and a resulting matrix
\[
P = \left( \begin{array}{ccc}
0 & 4 & 0 \\
0 & 0 & 0 \\
3 & 3 & 0 \end{array} \right).
\]
This leads to the ranking $C \succ A \succ B$, and scores $(7, 4, 0)$ for alternatives $C$, $A$, and $B$, respectively.

The Schulze method is Condorcet consistent {\it and} clone consistent. The computation of the strongest paths between all candidates is similar to computing the widest path problem on a graph and can be done using dynamic programming; in particular, using a variants of the Floyd-Warshall algorithm~\citep{Rosen03} in $O(m^3)$.

\subsubsection{Ranked pairs (Tideman's method)}

Ranked pairs~\citep{Tideman87} is another method that is both Condorcet consistent and clone consistent.

Ranked pairs also builds a graph and then remove nodes from the graph to determine their final ranks. The first thing that is done is to sort all the pairs of alternatives $(a,b)$ in decreasing order of strength of victory $\delta(a,b)$, choosing only one combination if $\delta(a,b) = 0$. Denote this sorted sequence (of length ${m \choose 2}$) $R_p$.

Having $R_p$, a directed (acyclic) graph is constructed. Enumerating each entry $(a, b) \in R_p$ in decreasing order of preference strength, we add a directed edge $a \rightarrow b$ only if it would not create a (directed) cycle. 

Once the graph is created, the source (a vertex with no incoming edges) corresponds to the winner. To determine a full ranking, we iteratively remove vertices (and all associated edges) that correspond to the sources (choosing arbitrarily if there are no remaining sources, \ie disconnected nodes). 

Again, there is no scores defined for this method, but we define the score of an alternative to be the total sum of all edge weights ($\delta(a,b)$) over all paths that reach other alternatives that originate from the source when it is removed.

In the running example from Figure~\ref{fig:meeple_pentathlon}, this leads to the following sorted pairs:
\begin{center}
\begin{tabular}{c|c}
{\bf $\delta(a,b)$} & Pair $((a,b))$\\
\hline
$3$ & $(A, B)$ \\
$1$ & $(C, A)$ \\
$1$ & $(C, B)$ \\
\hline
\end{tabular}
\end{center}
and a graph with edges $A \rightarrow B$, then $C \rightarrow A$, and $C \rightarrow B$, depicted in Figure~\ref{fig:rps_graph}.

\begin{figure}[h!]
    \centering
    \includegraphics[scale=0.2]{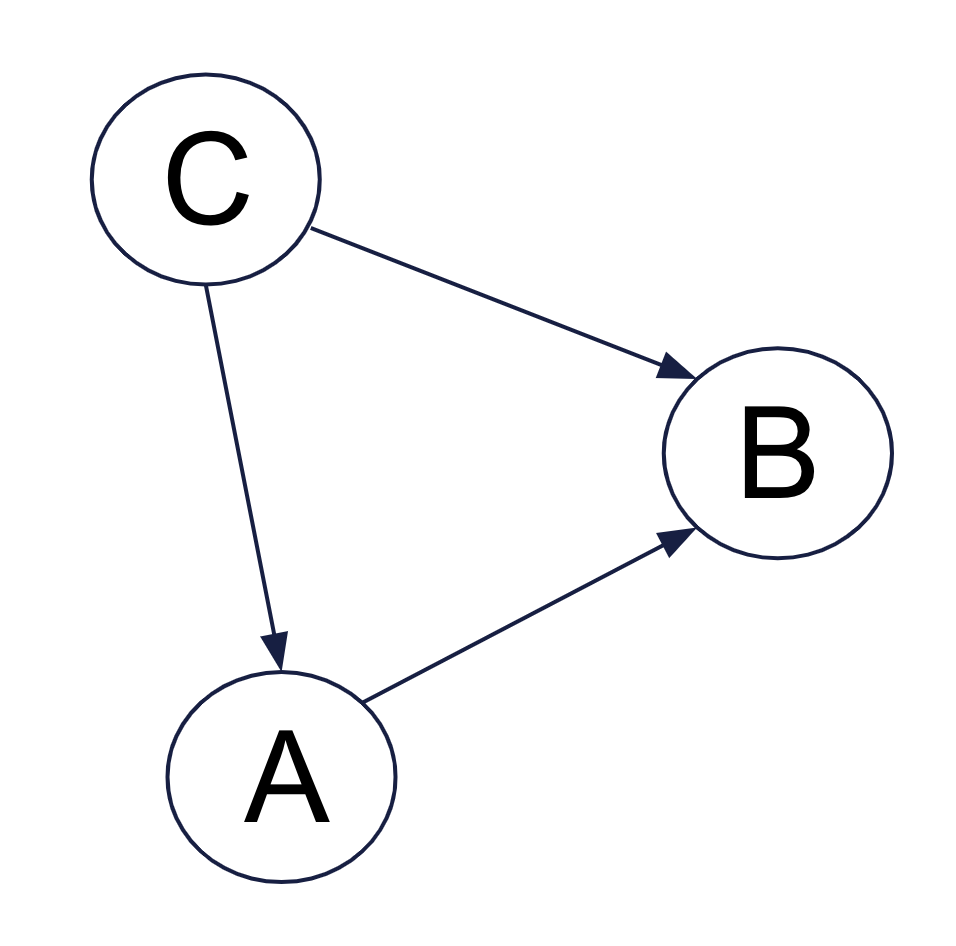}
    \caption{Graph constructed by ranked pairs method on the example from Figure~\ref{fig:meeple_pentathlon}.}
    \label{fig:rps_graph}
\end{figure}

So, first $C$ is removed, then $A$, leaving $B$,  which results in a ranking of $C \succ A \succ B$ and scores $(5, 3, 0)$ for $C$, $A$, and $B$, respectively.

\subsubsection{Maximal Lotteries}
\label{app:maximal-lotteries}

Maximal lotteries is probabilistic social choice method which chooses ranks based on the 
the solution of a two-player zero-sum game whose utilities are defined by the margin matrix~\citep{Krewaras65,Fishburn84}. That is, suppose there are probability distributions over alternatives (lotteries) $x, y \in \Delta(A)$. A lottery $x$ is maximal if and only if is is the solution of the game
\begin{equation}
\max_{x \in \Delta(A)} \min_{y \in \Delta(Y)} x^T M y \ge 0.
\label{eq:ml-game}
\end{equation}

The interpretation is that $x$ is a distribution over alternatives such that, if an adversary were to choose a different distribution, the expected margin would still be non-negative. 

Here, $M$ is an anti-symmetric matrix ($M = -M^T$), so the game above is a symmetric two-player zero-sum game over the same strategy spaces ($A$). As a result, It is a well-known fact of game theory that there always exists a solution to the game and the value of Equation~\ref{eq:ml-game} is always 0.
Also, an exact solution can be computed efficiently using linear programming that is polynomial in the size of $M$.

Maximal lotteries guarantees all of the consistency properties mentioned in this paper: Condorcet consistency, clone consistency, population consistency, and agenda consistency. 
It is the only method to satisfy both population consistency and composition consistency~\citep{brandl2016consistent} and can achieve this combination of properties due to being a probabilistic method. 
In addition, it satisfies independence of irrelevant alternatives~\citep{brandl2020arrovian}, a property at the heart of Arrow's impossibility theorem. Finally, by being probabilistic, it avoids some strategic manipulation that is known to be problematic in deterministic voting rules~\citep{Brandt17Rolling}.

In the recurring example from Figure~\ref{fig:meeple_pentathlon}, one can see from the margin matrix that $C$ is the only strategy that survives iterated elimination of dominated strategies. Hence, C is the only minimax-optimal solution to game formulated in Equation~\ref{eq:ml-game} so maximal lotteries chooses $C$ with probability 1.\\

\noindent {\bf Iterative Maximal Lotteries (IML)}\\

We now describe a slight refinement of maximal lotteries (ML) that produces a special type of social welfare function. The maximal lotteries method can be iterated as shown in Algorithm~\ref{alg:iterative-ml}.

\begin{algorithm}
\DontPrintSemicolon 
\KwIn{An initial set $A_0$ of alternatives}
\KwOut{A ranking of winner sets, $W$, and their corresponding probability distributions}
Let $W \gets ()$ be an empty sequence \;
\For{round $t \in \{0, 1, 2, \cdots, m-1\}$}{
    \If{$|A_t| = 0$}{
        {\bf break} \;
    }
    $x_t \gets \textsc{MaximalLotteries}(A_t)$ \;
    $w_t \gets \{ a \in A_t: x_t(a) > 0 \}$ \;
    Append $(x_t, w_t)$ to the end of $W$ \;
    $A_{t+1} = A_t - w_t$ \; \label{alg:reduce-alts}
    Remove the winners $w_t$ from all the votes. \; \label{alg:mod-votes}
}
\Return{$W$} \;
\caption{Iterative Maximal Lotteries (IML)}
\label{alg:iterative-ml}
\end{algorithm}

We refer to the cardinality $|W|$ as the {\it number of levels} of the resulting evaluation.
On each round, $W$ is augmented with the set of winners for that round (and their corresponding distributions), and then these agents are removed from the set of alternatives and deleted from all of the votes for the next round. (Note: in practice, this can be easily achieved by maintaining a margin matrix $M_t$ in each round and deleting the rows and columns that correspond to the winners in it rather than modifying the votes on line~\ref{alg:mod-votes}).

There are a some interesting interpretations of $W$. First, if $|W| = m$, then every $w_t$ is a singleton and there were no mixed strategy equilibria (cycles) and the order is strictly transitive. More generally, the number of levels determines how often (on average) the subgroups have mixed strategy (cyclic) solutions: a high value close to $m$ means mostly transitive solutions in each subgroup. A low number of levels (closer to 0) means there is often mixed strategy solutions.
Also, roughly speaking, the agent's level (the level achieved by that agent -- earlier in W referring to the top, and later to the bottom) determines its relative rank. An agent with a higher level means it was chosen earlier as a winner of the subgroup of agents at that time, whereas agents at the same level are comparable (contained within the same strategic cycle).

Define $\textsc{Level}(W, a) = 0$ for an agent at the bottom level and $|W| - 1$ for the agents at the top level. 
We extract a total order by defining a score that corresponds to the agent's level plus the probability of the strategy that selects it within its level:
\[
\textsc{IMLScore}(W, a) = \textsc{Level}(W, a) + x_{|W| - 1 - \textsc{Level}(W, a)}(a).
\]

This leads to an interpretable score. First, the number of levels is directly conveyed in the range of values (if the highest score is $s$, then the number of levels in the ranking ($|W|$) is $\lceil s \rceil + 1$), and the size of each level as well as probabilities assigned to each agent, directly readable from the second (fractional) part of the score. For example, an agent achieving a score of 3.45 means that it reached the third level and had a probability of $0.45$ among the level-4 agents (level 0 scores use the interval $(0,1]$, level 1 uses the internal $(1, 2]$, level 2 uses the interval $(2, 3]$, etc.).

In the recurring example from Figure~\ref{fig:meeple_pentathlon}, this leads to three rounds, with margin matrices:

\[
M_0 = \left( \begin{array}{ccc}
0 & 3 & -1 \\
-3 & 0 & -1 \\
1 & 1 & 0 \end{array} \right)
\]

\[
M_1 = \left( \begin{array}{cc}
0 & 3  \\
-3 & 0 \end{array} \right)
\]

\[
M_2 = ( 0 ),
\]

with solutions $x_0 = \{(C, 1)\}$, $x_1 = \{(A, 1)\}$, and $x_2 = \{(B, 1)\}$. 

This leads to $C \succ A \succ B$ and scores $(3, 2, 1)$ for alternatives $C$, $A$, and $B$, respectively. Note that the number of levels is accessible by inspection and there were no game-theoretic cycles in the iterative winner selection because all the values scores have integer values.

\section{Similarities and Differences between Maximal Lotteries and Nash Averaging}
\label{app:ml-vs-nashavg}

On first glance, it may appear that maximal lotteries~\citep{Fishburn84} and Nash averaging~\citep{Balduzzi18ReEval} are similar, and indeed there are some similarities: both solve a two-player zero-sum game derived from what can be viewed as preference information from the evaluation process. Namely, cardinal utilities for Nash averaging and ordinal information in the case of maximal lotteries. 

However, there are many differences between the two methods:
\begin{enumerate}
\item The game that is solved is different. Maximal lotteries solves the game induced by the margin matrix, $M$, which can be formulated from ranked preferences lists (or pairwise comparison data). Nash averaging solves the game that uses actual payoffs (\eg normalized scores in the case of Atari, and expected winnings in the AvA case like heads-up poker).
\item The targeted solution concept is different. The form of the Nash equilibrium returned by maximal lotteries is left unspecified (unconstrained) whereas Nash averaging requires the maximum-entropy Nash equilibrium to guarantee uniqueness. 
\item How rankings are derived from the solution is different. The rankings derived by maximal lotteries are based on the Nash equilibrium distribution itself, whereas the ranking according to Nash averaging is based on the expected value to each agent under the maximum-entropy Nash equilibrium.
\item Extensions to more general AvA setting (\ie $n$-player with $n \ge 3$, or general-sum) differ. Nash averaging naturally extends to the meta-game played by the $n$ players~\citep{marris2022game}, leading to tensors whose sizes are generally exponential in $n$ (necessitating more efficient solution concepts such as correlated equilibria). In maximal lotteries, using the AvA model described above,
every game played is a vote that contributes to defining the margin matrix, $M$, which is necessarily symmetric and hence quadratic in $m$.
\end{enumerate}

One of the main differences is that Nash averaging, by design, uses the magnitude of the score in comparing agents, so one big expected win of $10$ is preferred to seven wins of size $1$. In maximal lotteries, the wins are non-numerical, and it is {\it only} the number of times an agent ranks more highly than another that determines its pairwise value in the game that gets solved. The comparison may be determined from some underlying numerical value (such as number of supply centers achieved, in the game of Diplomacy), but only the pairwise comparisons end up influencing the margin matrix.

Maximal lotteries can be iterated while preserving its guarantees, as discussed in Section~\ref{app:maximal-lotteries}.
We believe that Nash averaging can be iterated in a very similar way (\eg as Algorithm~\ref{alg:iterative-ml}), producing a sequence of subgames where agents in the support are removed and process is continued in a new subgame without the agents. Similarly to maximal lotteries, this would produce a sequence of groups of agents where agents in the same group were all in the support of the subgame defined at the time they were removed. We believe that the invariance properties of the original Nash averaging~\citep{Balduzzi18ReEval} could also be preserved by this variant.
However, agents in the same group would always be indistinguishable because they achieve the same expected value to the maximum-entropy Nash equilibrium of the subgame that was solved when they were in the support.

\section{Voting Method Logistics: Variable Vote Lengths, Uneven Alternative Distributions, and Tie-Breaking}
\label{app:logistics}

There are several subtleties to be aware of when applying voting methods for evaluation. In this section, we discuss some of these issues.

Define the length of a vote to be the number of alternatives in its preference order. Some methods are can be sensitive to unequal vote lengths. For example, if Borda assigns points decreasing from $m-1$ to the lowest agent, then agents in votes of length less than $m$ are getting the same score despite being compared to fewer agents.
Similarly for other scoring methods: a first-place win for plurality is scored the same way even if, for example, in one case the ballots might contain only 2 candidates, and in another there are 7. 
Our applications, in Section~\ref{sec:eval}, were mostly unaffected by variable vote lengths as most of them had the same length. 
The only application that differed in length was the second application of HELM. 

Another potential problem is if the data set of votes contains an uneven distribution of alternatives. This can happen if, for example, one alternative is present in 80\% of the votes versus another alternative appearing in only 60\% of the votes. This, again, can affect scoring rules because one agent may have an easier time being reflected in the top ranks if they are present in a large number of votes. 
Again, most of our applications in Section~\ref{sec:eval} are not affected by this problem because the distribution is uniform or very close to uniform.
In the AvA application to Diplomacy described in Section~\ref{sec:eval-ava}, the average number of games player per player is $4.104 \pm 0.749$ (95\% confidence intervals); hence, this can have an effect on the results. However, since Elo does not control for uneven distributions and to remain fair to Elo, our results did not control for the distribution of games played. 

Another potential problem is how to deal with ties in scores, leading to uncertainty on how rank agents. In the applications in this paper, we simple break ties randomly in way that is consistent across all methods (more specifically, by adding small random perturbations to the values compared). However, there has been some work on more principled ways to break ties that could be applied to VasE. For example, the Normal Independent Voter Perturbation (NIVP) is a general method that satisfies several desirable properties, one of them being the preservation of clone-proofness~\citep{Freeman15General}.\\


\newpage
\clearpage


\section{Full Results on the Arcade Learning Environment}
\label{app:ale}


\subsubsection{Results from ~\citep[Table 5]{Hessel17Rainbow}}

Elo ratings for the ALE agents are shown on the left side of Table~\ref{tab:ale_elos}.

The VasE rankings and scores from~\citep[Table 5]{Hessel17Rainbow} are shown in Table~\ref{tab:atari-tbl5-full}. Note that in the case of Iterative Maximal Lotteries, there is a cycle in the preferences on the fourth level (between \textsc{a3c} and \textsc{dueling-DQN}) due to a tie in the head-to-head comparison, \ie $\delta(\textsc{a3c}, \textsc{deuling-dqn}) = 0$.

\begin{table}[ht!]
\centering
\begin{tabular}{cc}
{\bf Agent} & {\bf Elo rating} \\
\hline
\textsc{rainbow} & 424\\
\textsc{dist-dqn} & 325\\
\textsc{prio-ddqn} & 274\\
\textsc{duel} & 238\\
\textsc{a3c}   & 213\\
\textsc{ddqn} & 162 \\
\textsc{noisy-dqn} & 92\\
\textsc{dqn} & 0\\
\hline
\end{tabular}
~~~~~~~~~~
\begin{tabular}{cc}
{\bf Agent} & {\bf Elo rating} \\
\hline
\textsc{rainbow} & 1049\\
\textsc{dist-dqn} & 953\\
\textsc{prio-ddqn} & 912\\
\textsc{duel} & 880\\
\textsc{a3c}   & 855\\
\textsc{human} & 849 \\
\textsc{ddqn} & 810 \\
\textsc{noisy-dqn} & 748\\
\textsc{dqn} & 658\\
\textsc{random} & 0\\
\hline
\end{tabular}
\caption{Left: Elo ratings for ALE agents (relative to bottom agent defined to have rating 0) across 54 Atari games.
Right: Relative Elo ratings when human and random are included in the set of agents.
\label{tab:ale_elos}}
\end{table}

\begin{table*}[ht!]
\vspace{5cm}
\begin{center}
\begin{tabular}{|c|ll|}
\multicolumn{3}{c}{\bf Approval(k=3)}\\
\hline
Rank & Agent & Score\\
\hline
1 & rainbow & 41\\
2 & dist-dqn & 35\\
3 & prio-ddqn & 23\\
4 & a3c & 22\\
5 & duel-ddqn & 19\\
6 & ddqn & 11\\
7 & nois-dqn & 8\\
8 & dqn & 3\\
\hline
\end{tabular}
~
\begin{tabular}{|c|ll|}
\multicolumn{3}{c}{\bf Borda}\\
\hline
Rank & Agent & Score\\
\hline
1 & rainbow & 295\\
2 & dist-dqn & 247\\
3 & prio-ddqn & 222\\
4 & dueling-ddqn & 201\\
5 & a3c & 187\\
6 & ddqn & 159\\
7 & nois-dqn & 122\\
8 & dqn & 79\\
\hline
\end{tabular}
~
\begin{tabular}{|c|ll|}
\multicolumn{3}{c}{\bf Copeland}\\
\hline
Rank & Agent & Score\\
\hline
1 & rainbow & 7.0\\
2 & dist-dqn & 6.0\\
3 & prio-ddqn & 5.0\\
4 & duel-ddqn & 3.5\\
5 & a3c & 3.5\\
6 & ddqn & 2.0\\
7 & nois-dqn & 1.0\\
8 & dqn & 0.0\\
\hline
\end{tabular}\\

\vspace{0.5cm}

\begin{tabular}{|c|ll|}
\multicolumn{3}{c}{\bf Kemeny-Young}\\
\hline
Rank & Agent & Score\\
\hline
1 & rainbow & 295\\
2 & dist-dqn & 232\\
3 & prio-ddqn & 188\\
4 & a3c & 125\\
5 & duel-ddqn & 123\\
6 & ddqn & 78\\
7 & nois-dqn & 36\\
8 & dqn & 0\\
\hline
\end{tabular}
~
\begin{tabular}{|c|ll|}
\multicolumn{3}{c}{\bf Iterative Maximal Lotteries}\\
\hline
Rank & Agent & Score\\
\hline
1 & rainbow & 7.00\\
2 & dist-dqn & 6.00\\
3 & prio-ddqn & 5.00\\
4 & a3c & 3.98\\
5 & duel-ddqn & 3.02\\
6 & ddqn & 3.00\\
7 & nois-dqn & 2.00\\
8 & dqn & 1.00\\
\hline
\end{tabular}
~
\begin{tabular}{|c|ll|}
\multicolumn{3}{c}{\bf Plurality}\\
\hline
Rank & Agent & Score\\
\hline
1 & rainbow & 19\\
2 & a3c & 12\\
3 & dist-dqn & 8\\
4 & prio-ddqn & 6\\
5 & duel-ddqn & 5\\
6 & nois-dqn & 2\\
7 & ddqn & 2\\
8 & dqn & 0\\
\hline
\end{tabular}\\

\vspace{0.5cm}

\begin{tabular}{|c|ll|}
\multicolumn{3}{c}{\bf Ranked Pairs}\\
\hline
Rank & Agent & Score\\
\hline
1 & rainbow & 642\\
2 & dist-dqn & 430\\
3 & prio-ddqn & 292\\
5 & a3c & 102\\
4 & duel-ddqn & 152\\
6 & ddqn & 68\\
7 & nois-dqn & 20\\
8 & dqn & 0\\
\hline
\end{tabular}
~
\begin{tabular}{|c|ll|}
\multicolumn{3}{c}{\bf Schulze}\\
\hline
Rank & Agent & Score\\
\hline
1 & rainbow & 241\\
2 & dist-dqn & 204\\
3 & prio-ddqn & 168\\
4 & a3c & 137\\
5 & duel-ddqn & 110\\
6 & ddqn & 73\\
7 & nois-dqn & 36\\
8 & dqn & 0\\
\hline
\end{tabular}
~
\begin{tabular}{|c|ll|}
\multicolumn{3}{c}{\bf STV(num winners = 3)}\\
\hline
Rank & Agent & Score\\
\hline
1 & rainbow & 16.19\\
2 & dist-dqn & 15.14\\
3 & a3c & 14.16\\
4 & duel-ddqn & 8.12\\
5 & prio-ddqn & 7.6\\
6 & nois-dqn & 6.3\\
7 & ddqn & 5.2\\
8 & dqn & 4.0\\
\hline
\end{tabular}\\
\caption{VasE results for Atari data from~\citep[Table 5]{Hessel17Rainbow}. \label{tab:atari-tbl5-full}}
\end{center}
\vspace{1cm}
\end{table*}

\subsubsection{Results from ~\citep[Table 5]{Hessel17Rainbow} and Human + Random from ~\citep[Section H.4]{Badia20Agent57}}

If we add the scored for two more agents (random and human), the Elo ratings are shown in the right side of Table~\ref{tab:ale_elos}. The resulting VasE rankings and scores are shown in Table~\ref{tab:atari_two_extra_full}.

\begin{table*}[h!]
\begin{center}
\begin{tabular}{|c|ll|}
\multicolumn{3}{c}{\bf Approval(k=3)}\\
\hline
Rank & Agent & Score\\
\hline
1 & rainbow & 38\\
2 & dist-dqn & 29\\
3 & human & 23\\
4 & prio-ddqn & 20\\
5 & a3c & 20\\
6 & duel-ddqn & 13\\
7 & noisy-dqn & 8\\
8 & ddqn & 8\\
9 & dqn & 3\\
10 & random & 0\\
\hline
\end{tabular}
~
\begin{tabular}{|c|ll|}
\multicolumn{3}{c}{\bf Borda}\\
\hline
Rank & Agent & Score\\
\hline
1 & rainbow & 384\\
2 & dist-dqn & 331\\
3 & prio-ddqn & 304\\
4 & duel-ddqn & 284\\
5 & a3c & 268\\
6 & human & 264\\
7 & ddqn & 239\\
8 & nois-dqn & 200\\
9 & dqn & 152\\
10 & random & 4\\
\hline
\end{tabular}
~
\begin{tabular}{|c|ll|}
\multicolumn{3}{c}{\bf Copeland}\\
\hline
Rank & Agent & Score\\
\hline
1 & rainbow & 9.0\\
2 & dist-dqn & 8.0\\
3 & prio-ddqn & 7.0\\
4 & duel-ddqn & 5.5\\
5 & a3c & 5.5\\
6 & ddqn & 4.0\\
7 & human & 3.0\\
8 & nois-dqn & 2.0\\
9 & dqn & 1.0\\
10 & random & 0.0\\
\hline
\end{tabular}\\

\vspace{0.5cm}

\begin{tabular}{|c|ll|}
\multicolumn{3}{c}{\bf Kemeny-Young}\\
\hline
Rank & Agent & Score\\
\hline
1 & rainbow & 384\\
2 & dist-dqn & 314\\
3 & prio-ddqn & 271\\
4 & a3c & 206\\
5 & duel-ddqn & 206\\
6 & ddqn & 158\\
7 & human & 119\\
8 & nois-dqn & 90\\
9 & dqn & 54\\
10 & random & 0\\
\hline
\end{tabular}
~
\begin{tabular}{|c|ll|}
\multicolumn{3}{c}{\bf Iterative Maximal Lotteries}\\
\hline
Rank & Agent & Score\\
\hline
1 & rainbow & 9.00\\
2 & dist-dqn & 8.00\\
3 & prio-ddqn & 7.00\\
4 & duel-ddqn & 5.83\\
5 & a3c & 5.17\\
6 & ddqn & 5.00\\
7 & human & 4.00\\
8 & nois-dqn & 3.00\\
9 & dqn & 2.00\\
10 & random & 1.00\\
\hline
\end{tabular}
~
\begin{tabular}{|c|ll|}
\multicolumn{3}{c}{\bf Plurality}\\
\hline
Rank & Agent & Score\\
\hline
1 & human & 19\\
2 & rainbow & 11\\
3 & a3c & 9\\
4 & dist-dqn & 5\\
5 & duel-ddqn & 5\\
6 & prio-ddqn & 3\\
7 & ddqn & 2\\
8 & random & 0\\
9 & nois-dqn & 0\\
10 & dqn & 0\\
\hline
\end{tabular}\\

\vspace{0.5cm}

\begin{tabular}{|c|ll|}
\multicolumn{3}{c}{\bf Ranked Pairs}\\
\hline
Rank & Agent & Score\\
\hline
1 & rainbow & 1172\\
2 & dist-dqn & 890\\
3 & prio-ddqn & 696\\
5 & a3c & 390\\
4 & duel-ddqn & 444\\
6 & ddqn & 302\\
7 & human & 202\\
8 & nois-dqn & 126\\
9 & dqn & 54\\
10 & random & 0\\
\hline
\end{tabular}
~
\begin{tabular}{|c|ll|}
\multicolumn{3}{c}{\bf Schulze}\\
\hline
Rank & Agent & Score\\
\hline
1 & rainbow & 315\\
2 & dist-dqn & 278\\
3 & prio-ddqn & 242\\
4 & a3c & 211\\
5 & duel-ddqn & 184\\
6 & ddqn & 147\\
7 & human & 119\\
8 & nois-dqn & 90\\
9 & dqn & 54\\
10 & random & 0\\
\hline
\end{tabular}
~
\begin{tabular}{|c|ll|}
\multicolumn{3}{c}{\bf STV(num winners = 3)}\\
\hline
Rank & Agent & Score\\
\hline
1 & human & 20.19\\
2 & rainbow & 19.14\\
3 & a3c & 18.14\\
4 & dist-dqn & 10.12\\
5 & duel-ddqn & 9.6\\
6 & prio-ddqn & 8.5\\
7 & ddqn & 7.2\\
8 & random & 6.0\\
9 & nois-dqn & 5.0\\
10 & dqn & 4.0\\
\hline
\end{tabular}
\end{center}
\caption{VasE results on Atari with data from~\citep[Table 5]{Hessel17Rainbow} merged with two extra columns (human and random) taken from~\citep[Section H.4]{Badia20Agent57} \label{tab:atari_two_extra_full}.}
\vspace{2.5cm}
\end{table*}

\subsubsection{Analysis of Atari results with $\alpha$-Rank}

To try another game-theoretic approach, we also evaluate $\alpha$-rank on the same score matrix. 
In particular, we use the high ranking-intensity parameter method~\citep[Section 2.5]{Omidshafiei19AlphaRank} found in its open reference implementation~\citep{LanctotEtAl2019OpenSpiel}.
In contrast to Nash averaging, $\alpha$-rank finds a full support equilibrium with agent probabilities in $[0.07, 0.21]$, task probabilities in $[0.005, 0.057]$, and a ranking resembling VasE:
\textsc{rainbow} $\succ$ \textsc{dist-dqn} $\succ$ \textsc{duel} $\succ$ \textsc{prio-ddqn} $\succ$ \textsc{a3c} $\succ$  \textsc{ddqn} $\succ$ \textsc{noisy-dqn} $\succ$ \textsc{dqn}. In this case, $\alpha$-rank is considerably less adversarial in the weighting of tasks for ranking. The full weight distributions for agents and tasks found by $\alpha$-rank are given Figures \ref{fig:alpharank_agents} and \ref{fig:alpharank_tasks}.

\begin{figure*}[ht!]
\begin{center}
\hspace{-0.5cm}\includegraphics[scale=0.37]{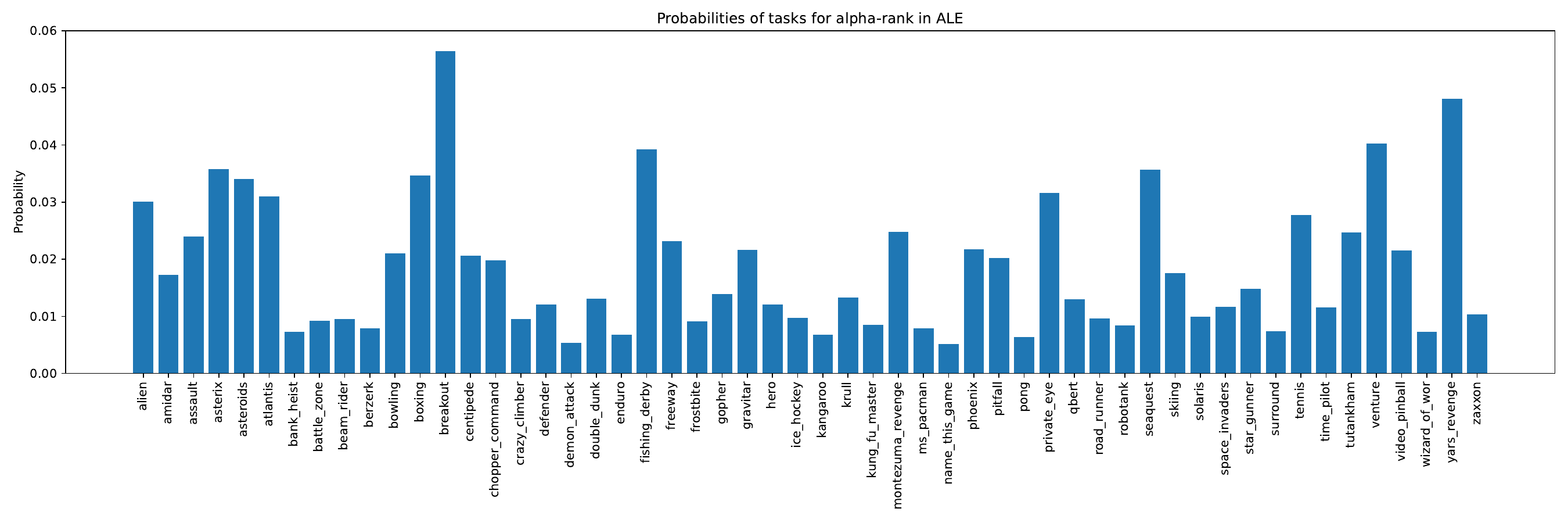}    
\end{center}
\caption{Task distribution computed by $\alpha$-Rank. \label{fig:alpharank_tasks}}
\end{figure*}

\begin{figure*}[ht!]
\begin{center}
\hspace{-5cm}\includegraphics[scale=0.7]{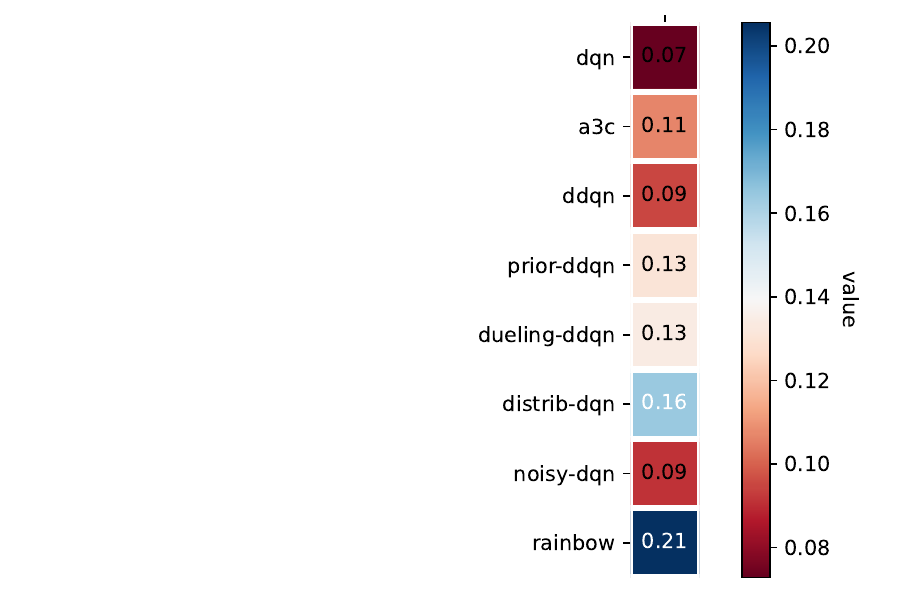}    
\caption{Agent distribution computed by $\alpha$-Rank.
\label{fig:alpharank_agents}}
\end{center}
\end{figure*}

\newpage
\clearpage

\section{Full Results on Adaptive Agents}
\label{app:ada}

\begin{figure*}[t]
\begin{tabular}{ccc}
\includegraphics[scale=0.46]{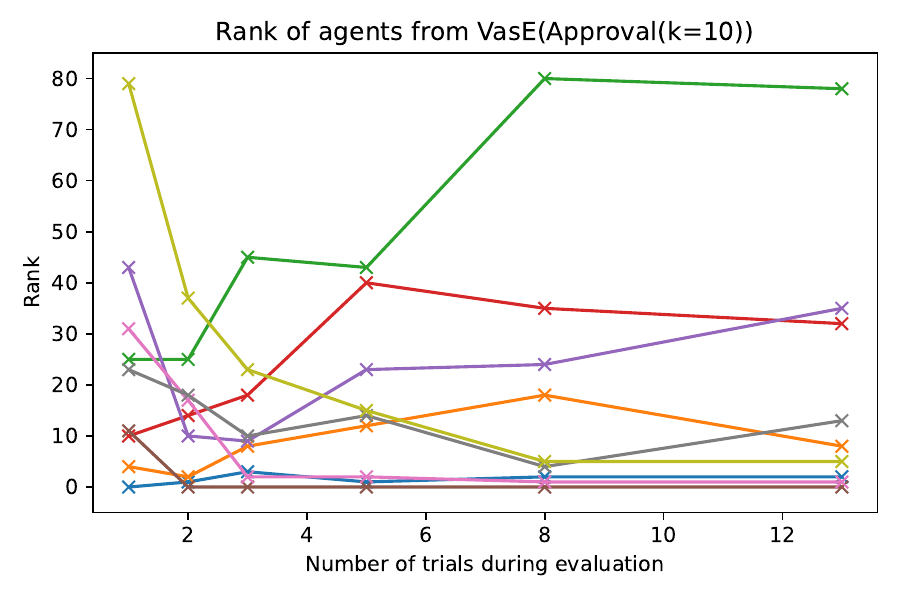} &
\includegraphics[scale=0.46]{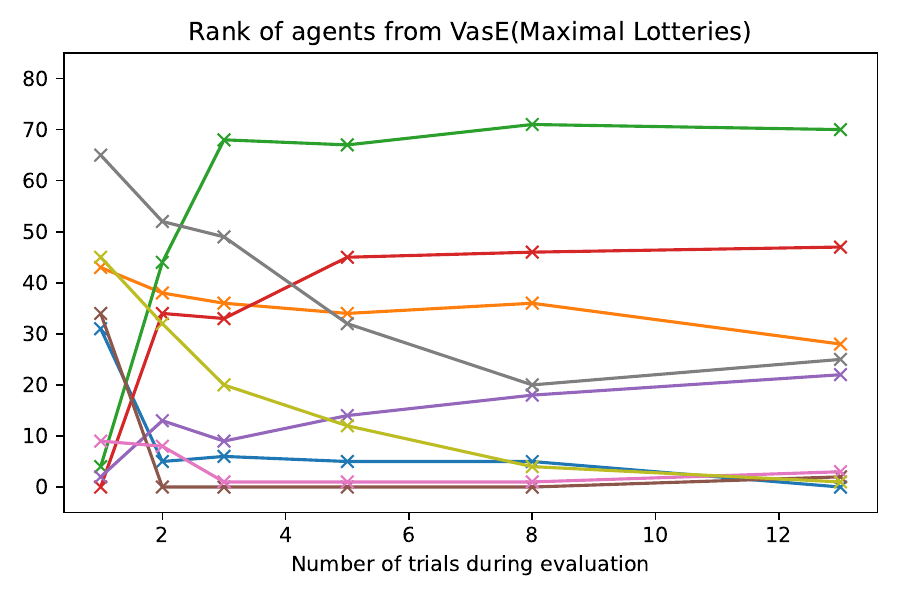} & 
\includegraphics[scale=0.32]{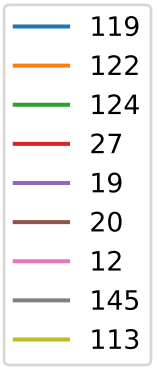}\\
  & & \\
\end{tabular}
\caption{Ranks of agents by number of test-time trials under different voting methods. A lower rank means a better agent. For $t = 1$, there is no Condorcet winner. For $t \in \{2, 3, 5, 8\}$ agent 20 is a strong Condorcet winner. At $t=13$: agent 119 is a weak Condorcet winner, Approval($k=10$)'s top three ranks are (20, 12, 119), and ML's top three are (119, 113, 20).
\label{fig:ada}}
\end{figure*}


We now apply VasE to a larger-scale environment: adaptive agents (AdA)~\citep{AdA23} in a procedurally-generated, open-ended 3D embodied task space (XLand)~\citep{XLand21}. 
Our main question: {\it How do agents compare in their ability to adapt to the task as the amount of test-time experience increases?}

We obtained a dataset from the authors' internal leaderboard consisting of $m = 150$ agents across $n = 46$ tasks.
Agents are trained via memory-based meta-RL resulting in human-timescale adaptation to new tasks. The set of agents differ by architecture and training parameters; the set of tasks differ in adaptation requirements such as experimentation, tool use, and division of labor.
During evaluation, each task is repeated for some number of successive trials $t \in \{1, 2, 3, 5, 8, 13\}$ where the environment (but {\it not} the agent memory) is reset between trials. 
A successfully adaptive agent may discover and remember relevant information and carry it to the future trials. 
The dataset is then a tensor of shape $(150, 46, 6)$ with each entry corresponding to the mean reward over 250 repetitions for that (agent, task, trials) combination.
We run six instances of VasE: one for each value of $t$, each instance of size $(150, 46)$. Figure~\ref{fig:ada} shows how the relative rank of a subset of agents changes as $t$ increases.

Our VasE framework uncovers several interesting phenomena.
First, we observe how significantly the ranks change between $t = 1$ and $t = 2$: the best agents at higher $t$ are not best at $t = 1$, possibly because the adaptive agents focus more on exploration to uncover the hidden rules of the task in early trials.
Similarly, some of the worst agents at higher values of $t$ are among the best at $t=1$, indicating a basic ability to perform on the task once but not as proficient in adaptation over several trials.
Some agents (12, 20, 119) remain among the best after $t \ge 2$ and others (113) gradually improve to among the best as $t$ increases, possibly due to different RL training techniques or memory approaches. 


We also ran Copeland on the same set of agents as in Figure~\ref{fig:ada}. The result is show in Figure~\ref{fig:ada-copeland}.

\begin{figure*}[t]
    \centering
\begin{tabular}{ccc}
\includegraphics[scale=0.65]{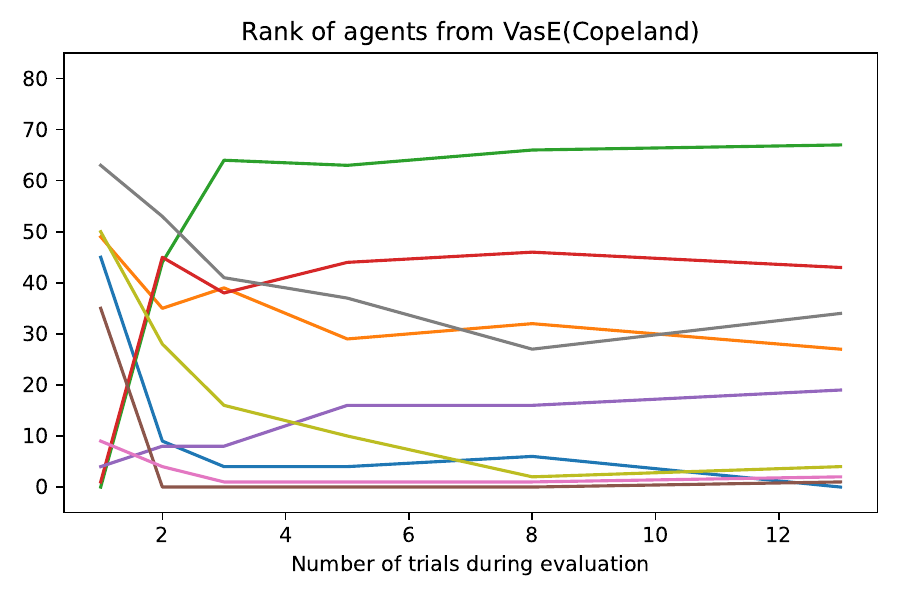}
& ~~~~ &
\includegraphics[scale=0.5]{legend.png}\\
\end{tabular}
\caption{Ranks of  agents  by  number  of  trials using the Copeland voting method. The result is similar to Iterative Maximal Lotteries from Figure~\ref{fig:ada}.}
\label{fig:ada-copeland}
\end{figure*}

In conversing with the original authors of~\citet{AdA23}, these were descriptions they sent us for the top agents:
\begin{itemize}
\item 12: An agent with a Transformer-XL memory, 169M Transformer-only parameters, (353M total parameters) and effective context length of 1800 timesteps, trained for 20 billion frames.
\item 20: A larger version of 12, trained for longer. Specifically with 256M Transformer-only parameters, (533M total parameters) trained for 46 billion frames.
\item 119: An agent that used "skip memory" as described in~\citep[Section 3.7]{AdA23}, combining a GRU and Transformer XL to increase the effective context length.
\item 113: Similar to 119 but with a shorter skip length (and therefore shorter effective context length), also trained for  fewer steps than 119.
\end{itemize}

The authors also told us that ``12 was the best agent used in the single-agent setting in \citet{AdA23}''.

Table~\ref{tab:ada-avg-ranks} shows the average ranks achieved by all the group of agents selected for Figures~\ref{fig:ada} and \ref{fig:ada-copeland} across the different number of trials. When using Copeland or Maximal Lotteries, agent 12 has had the best average rank among the chosen agents.

\begin{table*}[t!]
    \centering
    \begin{tabular}{c|ccc}
    Agent $\backslash$ Method &   Approval(k=10)  & Copeland & Iterative Maximal Lotteries \\
    \hline
     12   & 9.0       & \cellcolor{green!10} {\bf 3.0} & \cellcolor{green!10} {\bf 3.83} \\
     19   & 24.0      & 11.83     &  13.0       \\
     20   & 1.83      & 6.0       &  6.0        \\
     27   & 24.83     & 36.16     &  34.16      \\
     113  & 27.33     & 18.33     &  19.0       \\
     119  & \cellcolor{green!10} {\bf 1.5} & 11.33     &  8.66       \\
     122  & 8.66      & 35.16     &  35.83      \\
     124  & 49.33     & 50.66     &  54.0       \\
    \hline
    \end{tabular}
    \caption{Average rank for each agent over $t \in \{1, 2, 3, 5, 8, 12 \}$ (lower is better). }
    \label{tab:ada-avg-ranks}
\end{table*}

As a demonstration, the VasE results (for the top 10 agents only -- recall in this application that the number of agents, $m = 150$, and number of tasks, $n = 46$) for the number of trials $t=13$ are shown in Table~\ref{tab:ada-full}.

\begin{table*}[h!]
\vspace{3cm}
\begin{center}
\begin{tabular}{|c|ll|}
\multicolumn{3}{c}{\bf Approval(k=10)}\\
\hline
Rank & Agent & Score\\
\hline
1 & \textsc{20} & 25\\
2 & \textsc{12} & 20\\
3 & \textsc{119} & 19\\
4 & \textsc{103} & 15\\
5 & \textsc{35} & 15\\
6 & \textsc{113} & 14\\
7 & \textsc{118} & 12\\
8 & \textsc{21} & 12\\
9 & \textsc{122} & 11\\
10 & \textsc{120} & 11\\
\hline
\end{tabular}
~
\begin{tabular}{|c|ll|}
\multicolumn{3}{c}{\bf Borda}\\
\hline
Rank & Agent & Score\\
\hline
1 & \textsc{12} & 5513\\
2 & \textsc{20} & 5498\\
3 & \textsc{100} & 5396\\
4 & \textsc{118} & 5371\\
5 & \textsc{8} & 5313\\
6 & \textsc{31} & 5282\\
7 & \textsc{119} & 5270\\
8 & \textsc{112} & 5226\\
9 & \textsc{26} & 5211\\
10 & \textsc{113} & 5157\\
\hline
\end{tabular}
~
\begin{tabular}{|c|ll|}
\multicolumn{3}{c}{\bf Copeland}\\
\hline
Rank & Agent & Score\\
\hline
1 & \textsc{119} & 148.5\\
2 & \textsc{20} & 148.0\\
3 & \textsc{12} & 147.0\\
4 & \textsc{8} & 143.5\\
5 & \textsc{113} & 142.5\\
6 & \textsc{102} & 142.0\\
7 & \textsc{35} & 141.5\\
8 & \textsc{31} & 141.5\\
9 & \textsc{103} & 141.0\\
10 & \textsc{118} & 140.0\\
\hline
\end{tabular}\\

\vspace{0.5cm}

\begin{tabular}{|c|ll|}
\multicolumn{3}{c}{\bf Iterative Maximal Lotteries}\\
\hline
Rank & Agent & Score\\
\hline
1 & \textsc{119} & 64.96\\
2 & \textsc{113} & 64.04\\
3 & \textsc{20} & 64.00\\
4 & \textsc{12} & 63.00\\
5 & \textsc{8} & 61.53\\
6 & \textsc{103} & 61.19\\
7 & \textsc{102} & 61.14\\
8 & \textsc{31} & 61.14\\
9 & \textsc{35} & 60.98\\
10 & \textsc{30} & 60.02\\
\hline
\end{tabular}
~
\begin{tabular}{|c|ll|}
\multicolumn{3}{c}{\bf Plurality}\\
\hline
Rank & Agent & Score\\
\hline
1 & \textsc{119} & 8\\
2 & \textsc{113} & 4\\
3 & \textsc{145} & 3\\
4 & \textsc{103} & 3\\
5 & \textsc{121} & 2\\
6 & \textsc{118} & 2\\
7 & \textsc{35} & 2\\
8 & \textsc{26} & 2\\
9 & \textsc{20} & 2\\
10 & \textsc{146} & 1\\
\hline
\end{tabular}
~
\begin{tabular}{|c|ll|}
\multicolumn{3}{c}{\bf Ranked Pairs}\\
\hline
Rank & Agent & Score\\
\hline
1 & \textsc{119} & 244818\\
2 & \textsc{20} & 241132\\
3 & \textsc{12} & 236988\\
4 & \textsc{8} & 232798\\
5 & \textsc{102} & 228972\\
6 & \textsc{113} & 225482\\
7 & \textsc{35} & 221996\\
8 & \textsc{103} & 218860\\
9 & \textsc{31} & 215936\\
10 & \textsc{118} & 212186\\
\hline
\end{tabular}\\

\vspace{0.5cm}

\begin{tabular}{|c|ll|}
\multicolumn{3}{c}{\bf Schulze}\\
\hline
Rank & Agent & Score\\
\hline
1 & \textsc{119} & 3717\\
2 & \textsc{20} & 3693\\
3 & \textsc{12} & 3663\\
4 & \textsc{8} & 3631\\
5 & \textsc{102} & 3607\\
6 & \textsc{113} & 3581\\
7 & \textsc{35} & 3556\\
8 & \textsc{31} & 3534\\
9 & \textsc{100} & 3505\\
10 & \textsc{103} & 3480\\
\hline
\end{tabular}
~
\begin{tabular}{|c|ll|}
\multicolumn{3}{c}{\bf STV(num winners = 10)}\\
\hline
Rank & Agent & Score\\
\hline
1 & \textsc{119} & 300.8\\
2 & \textsc{20} & 299.5\\
3 & \textsc{103} & 298.5\\
4 & \textsc{145} & 297.5\\
5 & \textsc{113} & 296.5\\
6 & \textsc{21} & 295.6\\
7 & \textsc{121} & 294.5\\
8 & \textsc{35} & 293.6\\
9 & \textsc{143} & 292.6\\
10 & \textsc{33} & 150.3\\
\hline
\end{tabular}
\end{center}
\caption{VasE Results on the Adaptive Agents (\citep{AdA23}) leaderboard dataset. \label{tab:ada-full}}
\end{table*}

\newpage
\clearpage

\section{Evaluation of Large Language Models}
\label{app:language_models}

Here we begin to describe several applications of VasE to the evaluation of large language models.

In total, there are six applications on three data sources: Chatbot Arena~\citep{zheng2023judging}, AgentBench~\citep{liu2023agentbench}, our own prompt engineering task for text summarization and Q\&A, and three separate applications on the Holistic Evaluation of Language Models (HELM) data~\citep{liang2022holistic}.

\section{Full Results on AgentBench}
\label{app:summarization_domain}

The Elo ratings for agents in AgentBench are shown in Table~\ref{tab:elo-agentbench}.

As an example visualization that is possible with the VasE results, we show the subgraph of the top eight agents computed by the Ranked Pairs method in Figure~\ref{fig:agentbench-rps-graph}. A similar visualization could be computed for the graph constructed by the Schulze method.

\begin{table}[ht!]
\begin{center}
\begin{tabular}{lcc}
Model & Elo & AgentBench (OA) score\\
\hline
gpt-4                            &  1282   &  4.41\\
claude                           &   999   &  2.77\\
gpt-3.5-turbo                    &   992   &  2.55\\
text-davinci-003                 &   879   &   2.1\\
chatglm2                         &   692   &  1.31\\
claude-instant                   &   640   &   1.9\\
openchat-13b                     &   540   &  1.15\\
text-davinci-002                 &   525   &  1.46\\
t-bison-001                      &   496   &  1.39\\
wizardlm-30b                     &   491   &  0.83\\
llama2-13b-chat                  &   384   &  0.55\\
wizardlm-13b                     &   369   &  0.59\\
vicuna-13b                       &   361   &  0.62\\
codegeex2-6b                     &   354   &  0.53\\
openchat-8192-13b                &   336   &  0.51\\
baichuan-13b-chat                &   250   &  0.36\\
koala-13b                        &   227   &  0.34\\
llama2-7b-chat                   &   224   &  0.31\\
chatglm-6b                       &   209   &  0.31\\
vicuna-7b                        &   182   &  0.24\\
baichuan-7b                      &   182   &  0.22\\
wizardcoder-15b                  &   146   &  0.21\\
internlm-chat-7b                 &   125   &  0.23\\
dolly-v2-12b                     &   104   &  0.15\\
oasst-sft-4-pythia-12b           &     0   &  0.07\\
\end{tabular}
\end{center}
\caption{Elo values (relative to bottom agent who is defined to have 0 Elo) for agents in AgentBench benchmark evaluation. \label{tab:elo-agentbench}}
\end{table}

\begin{figure*}
\centering
\includegraphics[scale=0.5]{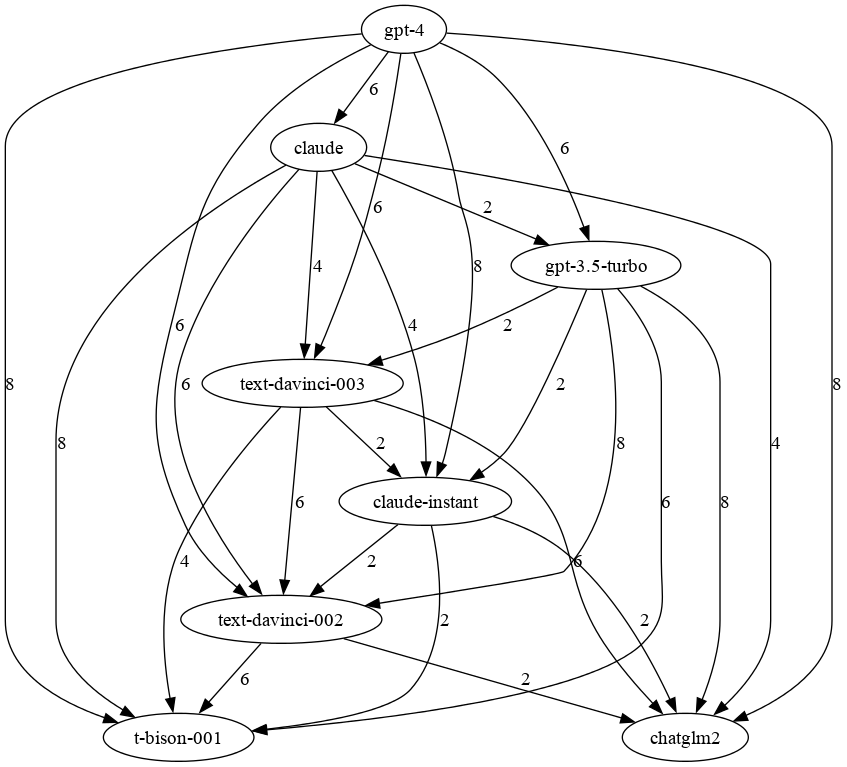}
\caption{The directed subgraph containing the top eight agents found by the Ranked Pairs method on the AgentBench data. Edge weights to the margin $\delta(x,y)$ for directed edge $x \rightarrow y$. From this visualization, see how strong (in number of votes (tasks)) each win or loss is against all neighbors.}
\label{fig:agentbench-rps-graph}
\end{figure*}

Full results for VasE methods are shown in Tables~\ref{tab:vase-agentbench-full1},
\ref{tab:vase-agentbench-full2}, 
\ref{tab:vase-agentbench-full3}, and \ref{tab:vase-agentbench-full4}.

\begin{table*}[t!]
\begin{center}
\begin{tabular}{|c|ll|}
\multicolumn{3}{c}{\bf Approval(k=8)}\\
\hline
Rank & Agent & Score\\
\hline
1 & {\tt gpt-4} & 8\\
2 & {\tt claude} & 8\\
3 & {\tt gpt-3.5-turbo} & 8\\
4 & {\tt text-davinci-003} & 7\\
5 & {\tt chatglm2} & 7\\
6 & {\tt claude-instant} & 6\\
7 & {\tt text-davinci-002} & 5\\
8 & {\tt t-bison-001} & 5\\
9 & {\tt openchat-13b} & 4\\
10 & {\tt wizardlm-30b} & 1\\
11 & {\tt vicuna-13b} & 1\\
12 & {\tt wizardlm-13b} & 1\\
13 & {\tt llama2-13b-chat} & 1\\
14 & {\tt codegeex2-6b} & 1\\
15 & {\tt chatglm-6b} & 1\\
16 & {\tt openchat-8192-13b} & 0\\
17 & {\tt baichuan-13b-chat} & 0\\
18 & {\tt koala-13b} & 0\\
19 & {\tt llama2-7b-chat} & 0\\
20 & {\tt vicuna-7b} & 0\\
21 & {\tt internlm-chat-7b} & 0\\
22 & {\tt baichuan-7b} & 0\\
23 & {\tt wizardcoder-15b} & 0\\
24 & {\tt dolly-v2-12b} & 0\\
25 & {\tt oasst-sft-4-pythia-12b} & 0\\
\hline
\end{tabular}
~
\begin{tabular}{|c|ll|}
\multicolumn{3}{c}{\bf Borda}\\
\hline
Rank & Agent & Score\\
\hline
1 & {\tt gpt-4} & 188\\
2 & {\tt claude} & 174\\
3 & {\tt gpt-3.5-turbo} & 173\\
4 & {\tt text-davinci-003} & 164\\
5 & {\tt chatglm2} & 146\\
6 & {\tt claude-instant} & 138\\
7 & {\tt text-davinci-002} & 120\\
8 & {\tt t-bison-001} & 115\\
9 & {\tt openchat-13b} & 115\\
10 & {\tt wizardlm-30b} & 111\\
11 & {\tt llama2-13b-chat} & 99\\
12 & {\tt vicuna-13b} & 94\\
13 & {\tt codegeex2-6b} & 89\\
14 & {\tt wizardlm-13b} & 87\\
15 & {\tt openchat-8192-13b} & 83\\
16 & {\tt llama2-7b-chat} & 67\\
17 & {\tt baichuan-13b-chat} & 66\\
18 & {\tt koala-13b} & 62\\
19 & {\tt chatglm-6b} & 57\\
20 & {\tt vicuna-7b} & 51\\
21 & {\tt baichuan-7b} & 51\\
22 & {\tt dolly-v2-12b} & 40\\
23 & {\tt wizardcoder-15b} & 38\\
24 & {\tt internlm-chat-7b} & 37\\
25 & {\tt oasst-sft-4-pythia-12b} & 35\\
\hline
\end{tabular}
\end{center}
\caption{Results for Approval($k=8$) and Borda methods on AgentBench. \label{tab:vase-agentbench-full1}}
\end{table*}

\begin{table*}[t!]
\begin{center}
\begin{tabular}{|c|ll|}
\multicolumn{3}{c}{\bf Copeland}\\
\hline
Rank & Agent & Score\\
\hline
1 & {\tt gpt-4} & 24.0\\
2 & {\tt claude} & 23.0\\
3 & {\tt gpt-3.5-turbo} & 22.0\\
4 & {\tt text-davinci-003} & 21.0\\
5 & {\tt claude-instant} & 20.0\\
6 & {\tt text-davinci-002} & 19.0\\
7 & {\tt chatglm2} & 17.5\\
8 & {\tt t-bison-001} & 16.5\\
9 & {\tt openchat-13b} & 16.5\\
10 & {\tt wizardlm-30b} & 15.0\\
11 & {\tt llama2-13b-chat} & 13.0\\
12 & {\tt codegeex2-6b} & 12.5\\
13 & {\tt vicuna-13b} & 11.5\\
14 & {\tt openchat-8192-13b} & 11.0\\
15 & {\tt wizardlm-13b} & 10.5\\
16 & {\tt baichuan-13b-chat} & 8.0\\
17 & {\tt koala-13b} & 8.0\\
18 & {\tt llama2-7b-chat} & 8.0\\
19 & {\tt vicuna-7b} & 6.0\\
20 & {\tt chatglm-6b} & 5.5\\
21 & {\tt baichuan-7b} & 4.5\\
22 & {\tt dolly-v2-12b} & 2.5\\
23 & {\tt internlm-chat-7b} & 2.0\\
24 & {\tt oasst-sft-4-pythia-12b} & 2.0\\
25 & {\tt wizardcoder-15b} & 0.5\\
\hline
\end{tabular}
~
\begin{tabular}{|c|ll|}
\multicolumn{3}{c}{\bf Iterative Maximal Lotteries}\\
\hline
Rank & Agent & Score\\
\hline
1 & {\tt gpt-4} & 17.00\\
2 & {\tt claude} & 16.00\\
3 & {\tt gpt-3.5-turbo} & 15.00\\
4 & {\tt text-davinci-003} & 14.00\\
5 & {\tt claude-instant} & 13.0\\
6 & {\tt text-davinci-002} & 12.00\\
7 & {\tt chatglm2} & 10.56\\
8 & {\tt t-bison-001} & 10.43\\
9 & {\tt openchat-13b} & 10.00\\
10 & {\tt wizardlm-30b} & 9.00\\
11 & {\tt codegeex2-6b} & 7.66\\
12 & {\tt llama2-13b-chat} & 7.34\\
13 & {\tt vicuna-13b} & 6.52\\
14 & {\tt openchat-8192-13b} & 6.48\\
15 & {\tt wizardlm-13b} & 5.53\\
16 & {\tt baichuan-13b-chat} & 5.47\\
17 & {\tt vicuna-7b} & 4.50\\
18 & {\tt koala-13b} & 4.36\\
19 & {\tt llama2-7b-chat} & 4.13\\
20 & {\tt baichuan-7b} & 3.57\\
21 & {\tt chatglm-6b} & 3.42\\
22 & {\tt dolly-v2-12b} & 2.52\\
23 & {\tt oasst-sft-4-pythia-12b} & 2.48\\
24 & {\tt internlm-chat-7b} & 2.00\\
25 & {\tt wizardcoder-15b} & 1.00\\
\hline
\end{tabular}
\end{center}
\caption{Results for Copeland and Iterative Maximal Lotteries methods on AgentBench. \label{tab:vase-agentbench-full2}}
\end{table*}

\begin{table*}[t!]
\begin{center}
\begin{tabular}{|c|ll|}
\multicolumn{3}{c}{\bf Plurality}\\
\hline
Rank & Agent & Score\\
\hline
1 & {\tt gpt-4} & 7\\
2 & {\tt gpt-3.5-turbo} & 1\\
3 & {\tt claude} & 0\\
4 & {\tt text-davinci-003} & 0\\
5 & {\tt claude-instant} & 0\\
6 & {\tt text-davinci-002} & 0\\
7 & {\tt t-bison-001} & 0\\
8 & {\tt chatglm2} & 0\\
9 & {\tt openchat-13b} & 0\\
10 & {\tt wizardlm-30b} & 0\\
11 & {\tt vicuna-13b} & 0\\
12 & {\tt wizardlm-13b} & 0\\
13 & {\tt llama2-13b-chat} & 0\\
14 & {\tt codegeex2-6b} & 0\\
15 & {\tt openchat-8192-13b} & 0\\
16 & {\tt baichuan-13b-chat} & 0\\
17 & {\tt koala-13b} & 0\\
18 & {\tt llama2-7b-chat} & 0\\
19 & {\tt chatglm-6b} & 0\\
20 & {\tt vicuna-7b} & 0\\
21 & {\tt internlm-chat-7b} & 0\\
22 & {\tt baichuan-7b} & 0\\
23 & {\tt wizardcoder-15b} & 0\\
24 & {\tt dolly-v2-12b} & 0\\
25 & {\tt oasst-sft-4-pythia-12b} & 0\\
\hline
\end{tabular}
~~~
\begin{tabular}{|c|ll|}
\multicolumn{3}{c}{\bf Ranked Pairs}\\
\hline
Rank & Agent & Score\\
\hline
1 & {\tt gpt-4} & 1516\\
2 & {\tt claude} & 1332\\
3 & {\tt gpt-3.5-turbo} & 1170\\
4 & {\tt text-davinci-003} & 1008\\
5 & {\tt claude-instant} & 860\\
6 & {\tt text-davinci-002} & 760\\
7 & {\tt t-bison-001} & 480\\
8 & {\tt chatglm2} & 612\\
9 & {\tt openchat-13b} & 482\\
10 & {\tt wizardlm-30b} & 408\\
11 & {\tt llama2-13b-chat} & 246\\
12 & {\tt vicuna-13b} & 140\\
13 & {\tt codegeex2-6b} & 214\\
14 & {\tt openchat-8192-13b} & 170\\
15 & {\tt wizardlm-13b} & 130\\
16 & {\tt koala-13b} & 84\\
17 & {\tt baichuan-13b-chat} & 68\\
18 & {\tt llama2-7b-chat} & 52\\
19 & {\tt chatglm-6b} & 18\\
20 & {\tt vicuna-7b} & 26\\
21 & {\tt baichuan-7b} & 20\\
22 & {\tt dolly-v2-12b} & 4\\
23 & {\tt internlm-chat-7b} & 2\\
24 & {\tt oasst-sft-4-pythia-12b} & 2\\
25 & {\tt wizardcoder-15b} & 0\\
\hline
\end{tabular}
\end{center}
\caption{Results for plurality and Ranked Pairs methods on AgentBench. \label{tab:vase-agentbench-full3}}
\end{table*}

\vspace{0.5cm}

\begin{table*}[t!]
\begin{center}
\begin{tabular}{|c|ll|}
\multicolumn{3}{c}{\bf Schulze}\\
\hline
Rank & Agent & Score\\
\hline
1 & {\tt gpt-4} & 123\\
2 & {\tt claude} & 116\\
3 & {\tt gpt-3.5-turbo} & 111\\
4 & {\tt text-davinci-003} & 106\\
5 & {\tt claude-instant} & 101\\
6 & {\tt text-davinci-002} & 96\\
7 & {\tt t-bison-001} & 89\\
8 & {\tt chatglm2} & 85\\
9 & {\tt openchat-13b} & 80\\
10 & {\tt wizardlm-30b} & 75\\
11 & {\tt llama2-13b-chat} & 69\\
12 & {\tt vicuna-13b} & 64\\
13 & {\tt codegeex2-6b} & 60\\
14 & {\tt openchat-8192-13b} & 55\\
15 & {\tt wizardlm-13b} & 50\\
16 & {\tt koala-13b} & 44\\
17 & {\tt baichuan-13b-chat} & 39\\
18 & {\tt llama2-7b-chat} & 34\\
19 & {\tt chatglm-6b} & 29\\
20 & {\tt vicuna-7b} & 25\\
21 & {\tt baichuan-7b} & 20\\
22 & {\tt dolly-v2-12b} & 14\\
23 & {\tt internlm-chat-7b} & 9\\
24 & {\tt oasst-sft-4-pythia-12b} & 5\\
25 & {\tt wizardcoder-15b} & 0\\
\hline
\end{tabular}
~~~
\begin{tabular}{|c|ll|}
\multicolumn{3}{c}{\bf STV($k=8$)}\\
\hline
Rank & Agent & Score\\
\hline
1 & {\tt gpt-4} & 50.7\\
2 & {\tt claude} & 49.4\\
3 & {\tt gpt-3.5-turbo} & 48.2\\
4 & {\tt text-davinci-003} & 47.2\\
5 & {\tt claude-instant} & 46.2\\
6 & {\tt text-davinci-002} & 45.2\\
7 & {\tt t-bison-001} & 44.1\\
8 & {\tt openchat-13b} & 43.1\\
9 & {\tt chatglm2} & 25.0\\
10 & {\tt wizardlm-30b} & 24.0\\
11 & {\tt vicuna-13b} & 23.0\\
12 & {\tt wizardlm-13b} & 22.0\\
13 & {\tt llama2-13b-chat} & 21.0\\
14 & {\tt codegeex2-6b} & 20.0\\
15 & {\tt openchat-8192-13b} & 19.0\\
16 & {\tt baichuan-13b-chat} & 18.0\\
17 & {\tt koala-13b} & 17.0\\
18 & {\tt llama2-7b-chat} & 16.0\\
19 & {\tt chatglm-6b} & 15.0\\
20 & {\tt vicuna-7b} & 14.0\\
21 & {\tt internlm-chat-7b} & 13.0\\
22 & {\tt baichuan-7b} & 12.0\\
23 & {\tt wizardcoder-15b} & 11.0\\
24 & {\tt dolly-v2-12b} & 10.0\\
25 & {\tt oasst-sft-4-pythia-12b} & 9.0\\
\hline
\end{tabular}
\end{center}
\caption{Results for Schulze and STV($k=8$) methods on AgentBench. \label{tab:vase-agentbench-full4}}
\end{table*}

\clearpage

\clearpage

\section{VasE Analysis on Chatbot Arena Dataset}

In this section, we apply VasE to the Chatbot Arena conversation dataset~\citep{zheng23chatbot_arena_dataset}\footnote{Dataset available at \url{https://lmsys.org/blog/2023-07-20-dataset/}, released on July 10th, 2023. Updated leaderboard being hosted at \url{https://huggingface.co/spaces/lmsys/chatbot-arena-leaderboard}}. This dataset contains 33000 crowd-source conversations between with human preference annotations from Chatbot Arena\footnote{\url{https://lmsys.org/blog/2023-05-03-arena/}}.

The data set contains conversations between 20 different models, containing some overlap with models in AgentBench.
Each datapoint contains binary comparisons (win, loss, ties) between two models, hence our application of VasE focuses on Condorcet voting methods.

Results for VasE are shown in Tables \ref{tab:vase-chatbot-arena-full1} and \ref{tab:vase-chatbot-arena-full2}.
The rankings are mostly consistent with the Elo ratings found in Chatbot Arena, but we identify a two interesting findings: IML finds some game-theoretic cycles and comparison of the Ranked Pairs graph with AgentBench.

\subsection{Comparing IML and other Condorcet Methods} 
\begin{figure}[t]
\begin{center}
$\begin{bmatrix}
 0  & 0   & 20   &  32 &  24 & -2  & 20 & 14  & 39 \\
 0  & 0   & 33   &  87 &  17 & -7  & 67 & -8  & 50 \\
-20 & -33 &  0   &  25 &  -7 &  2  & 10 & -28 & 12 \\
-32 & -87 & -25  &   0 & -21 & -3  & 26 & -54 & -11\\
-24 & -17 &  7   &  21 &   0 & -13 & 21 & -27 & 27\\
  2 &  7  & -2   &   3 &  13 &  0  &  1 &  6  & 5\\
-20 & -67 & -10  & -26 & -21 & -1  &  0 & -65 & -26\\
-14 &  8  & 28   &  54 &  27 & -6  & 65 &  0  & 48\\
-39 & -50 & -12  &  11 & -27 & -5  & 26 & -48 &  0\\
\end{bmatrix}$
\end{center}
\caption{Margin subgame solved by IML when there are 8 remaining alternatives. Entries represent $\delta(x, y)$ for agent on row $x$ and column $y$. Agents {\tt RWKV-4-Raven-14B}, {\tt chatglm}, and {\tt gpt4all-13b-snoozy} correspond to agents on the first, third, and sixth rows and columns, respectively. \label{fig:iml-chatbot-arena-subgame}}
\end{figure}

Here, we show an instance where IML identifies a game-theoretic cycle whereas the other Condorcet methods rank the alternatives differently. Looking at Table~\ref{tab:vase-chatbot-arena-full}, ranks 12-20: Copeland and Ranked Pairs and Schulze rank the following three agents according to:\\

\noindent {\tt gpt4all-13b-snoozy} $\succ$ {\tt RWKV-4-Raven-14B} $\succ$ {\tt chatglm-6b},\\

\noindent whereas IML ranks {\tt gpt4all-13b-snoozy} above {\tt RWKV-4-Raven-14B} and {\tt chatglm}, but all three within the same level (identifying a game-theoretic cycle).

Looking more closely, we see that IML find a distribution with probabilities $0.833$ on {\tt gpt4all-13b-snoozy}, $0.0833$ on {\tt RWKV-4-Raven-14B}, and $0.0833$ on {\tt chatglm-6b}, and a margin subgame depicted in Figure~\ref{fig:iml-chatbot-arena-subgame}.
This distribution leads to the following expected values to the column player:\\

$~~~~~0.8333~~[2 ~~ 7 ~~ -2 ~~ 3 ~~ 13 ~~ 0 ~~ 1 ~~ 6 ~~ 5 ]$\\

$~~~+0.0833~~[0 ~~ 0 ~~ 20 ~~ 32 ~~ 24 ~~ -2 ~~ 20 ~~ 14 ~~ 39 ]$\\

$~~~+0.0833 ~~ [-20 ~~ -33 ~~ 0 ~~ 25 ~~ -7 ~~ 2 ~~ 10 ~~ -28 ~~ 12 ]$\\

$=~~ [0~~3.0821~~0~~7.2471~~12.2451~~0~~3.332~~3.8318~~  8.4133]$.\\

\noindent When the column (minimizing) player plays a best response here they then achieve a value of 0, which is the equilibrium value (both players receive 0 in expectation). 

From this analysis, it is also easy to see that choosing any one of the three deterministically is not an equilibrium strategy.
If, instead, the row player chose only one among these three strategies to play with probability 1, the column player would have a better response in each case and force the row player to receive strictly lower payoff (\i.e. -2 if the row player chooses either {\tt gpt4all-13b-snoozy} or {\tt RWKV-4-Raven-14B}, or -33 of the row player chooses {\tt chatglm-6b}. Hence, IML is choosing a safer maximal lottery due to the adversarial nature of the margin subgame, identifying three strategies and a mixed distribution at this level.

\subsection{Visualization and Comparison to Agent Bench}

We also show a visualization of the Ranked Pairs graph on this dataset in Figure~\ref{fig:chatbot_arena_rps_graph}. Due to the much larger dataset, the edge weights are significantly more granular in this case than on AgentBench. However, we do see some similarities: top agents are similarly-ranked. In both cases {\tt gpt-4} $\succ$ {\tt claude} $\succ$ {\tt gpt-3.5-turbo}. This is a remarkable consistency given that the benchmarks test significantly different tasks, and one comes from score metrics and the other from crowd-sourced human annotations. 

\begin{table*}
\begin{center}
\begin{tabular}{|c|ll|}
\multicolumn{3}{c}{\bf Copeland}\\
\hline
Rank & Agent & Score\\
\hline
1 & {\tt gpt-4} & 19.0\\
2 & {\tt claude-v1} & 18.0\\
3 & {\tt claude-instant-v1} & 17.0\\
4 & {\tt guanaco-33b} & 16.0\\
5 & {\tt gpt-3.5-turbo} & 15.0\\
6 & {\tt wizardlm-13b} & 13.5\\
7 & {\tt palm-2} & 13.0\\
8 & {\tt vicuna-13b} & 12.0\\
9 & {\tt vicuna-7b} & 11.0\\
10 & {\tt koala-13b} & 10.5\\
11 & {\tt mpt-7b-chat} & 9.0\\
12 & {\tt gpt4all-13b-snoozy} & 7.0\\
13 & {\tt RWKV-4-Raven-14B} & 6.5\\
14 & {\tt oasst-pythia-12b} & 6.0\\
15 & {\tt alpaca-13b} & 5.5\\
16 & {\tt chatglm-6b} & 4.0\\
17 & {\tt fastchat-t5-3b} & 4.0\\
18 & {\tt stablelm-tuned-alpha-7b} & 2.0\\
19 & {\tt dolly-v2-12b} & 1.0\\
20 & {\tt llama-13b} & 0.0\\
\hline
\end{tabular}\\

\vspace{0.5cm}

\begin{tabular}{|c|ll|}
\multicolumn{3}{c}{\bf Iterative Maximal Lotteries}\\
\hline
Rank & Agent & Score\\
\hline
1 & {\tt gpt-4} & 17.00\\
2 & {\tt claude-v1} & 16.00\\
3 & {\tt claude-instant-v1} & 15.00\\
4 & {\tt guanaco-33b} & 14.00\\
5 & {\tt gpt-3.5-turbo} & 13.00\\
6 & {\tt wizardlm-13b} & 11.997\\
7 & {\tt koala-13b} & 11.003\\
8 & {\tt palm-2} & 11.00\\
9 & {\tt vicuna-13b} & 10.00\\
10 & {\tt vicuna-7b} & 9.00\\
11 & {\tt mpt-7b-chat} & 8.00\\
12 & {\tt gpt4all-13b-snoozy} & 6.83\\
13 & {\tt RWKV-4-Raven-14B} & 6.083\\
14 & {\tt chatglm-6b} & 6.083\\
15 & {\tt oasst-pythia-12b} & 6.00\\
16 & {\tt alpaca-13b} & 5.00\\
17 & {\tt fastchat-t5-3b} & 4.00\\
18 & {\tt stablelm-tuned-alpha-7b} & 3.00\\
19 & {\tt dolly-v2-12b} & 2.00\\
20 & {\tt llama-13b} & 1.00\\
\hline
\end{tabular}
\end{center}
\caption{VasE results for Copeland and IML methods on the Chatbot Arena conversation dataset. \label{tab:vase-chatbot-arena-full1}}
\end{table*}

\begin{table*}
\begin{center}
\begin{tabular}{|c|ll|}
\multicolumn{3}{c}{\bf Ranked Pairs}\\
\hline
Rank & Agent & Score\\
\hline
1 & {\tt gpt-4} & 11624\\
2 & {\tt claude-v1} & 9220\\
3 & {\tt claude-instant-v1} & 7421\\
4 & {\tt guanaco-33b} & 6350\\
5 & {\tt gpt-3.5-turbo} & 6146\\
6 & {\tt wizardlm-13b} & 4405\\
7 & {\tt palm-2} & 4207\\
8 & {\tt vicuna-13b} & 3721\\
9 & {\tt vicuna-7b} & 2126\\
10 & {\tt koala-13b} & 1801\\
11 & {\tt mpt-7b-chat} & 963\\
12 & {\tt gpt4all-13b-snoozy} & 856\\
13 & {\tt RWKV-4-Raven-14B} & 819\\
14 & {\tt oasst-pythia-12b} & 670\\
15 & {\tt alpaca-13b} & 440\\
16 & {\tt fastchat-t5-3b} & 186\\
17 & {\tt chatglm-6b} & 110\\
18 & {\tt stablelm-tuned-alpha-7b} & 63\\
19 & {\tt dolly-v2-12b} & 26\\
20 & {\tt llama-13b} & 0\\
\hline
\end{tabular}\\

\vspace{0.5cm}

\begin{tabular}{|c|ll|}
\multicolumn{3}{c}{\bf Schulze}\\
\hline
Rank & Agent & Score\\
\hline
1 & {\tt gpt-4} & 1190\\
2 & {\tt claude-v1} & 1071\\
3 & {\tt claude-instant-v1} & 1021\\
4 & {\tt guanaco-33b} & 969\\
5 & {\tt gpt-3.5-turbo} & 939\\
6 & {\tt wizardlm-13b} & 896\\
7 & {\tt palm-2} & 859\\
8 & {\tt vicuna-13b} & 784\\
9 & {\tt vicuna-7b} & 713\\
10 & {\tt koala-13b} & 639\\
11 & {\tt mpt-7b-chat} & 541\\
12 & {\tt gpt4all-13b-snoozy} & 518\\
13 & {\tt RWKV-4-Raven-14B} & 499\\
14 & {\tt oasst-pythia-12b} & 419\\
15 & {\tt alpaca-13b} & 304\\
16 & {\tt fastchat-t5-3b} & 221\\
17 & {\tt chatglm-6b} & 170\\
18 & {\tt stablelm-tuned-alpha-7b} & 120\\
19 & {\tt dolly-v2-12b} & 63\\
20 & {\tt llama-13b} & 0\\
\hline
\end{tabular}
\end{center}
\caption{VasE results for Ranked Pairs and Schulze methods on the Chatbot Arena conversation dataset. \label{tab:vase-chatbot-arena-full2}}
\end{table*}

\begin{figure*}
\centering
\includegraphics[scale=0.5]{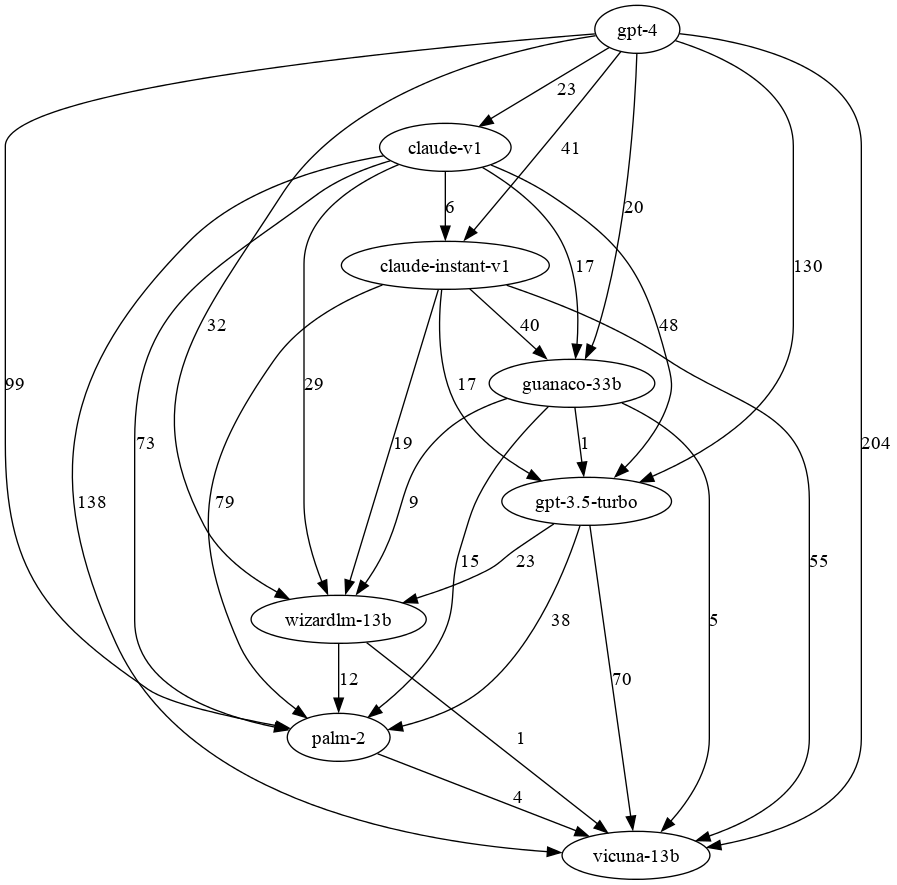}
\caption{The directed subgraph containing the top eight agents found by the Ranked Pairs method on the Chatbot Arena data. Edge weights to the margin $\delta(x,y)$ for directed edge $x \rightarrow y$. From this visualization, see how strong (in number of votes) each win or loss is against all neighbors.}
\label{fig:chatbot_arena_rps_graph}
\end{figure*}


\section{Automated Prompt Evaluation for Language Model Task Summarization}
\label{app:summarization_domain}


We apply our evaluation method to automate the process of selecting task prompts for large language models (LLMs). LLMs achieve high performance on a wide range of language tasks, ranging from sentiment analysis or entity recognition to summarization and paraphrasing. However, applying LLMs to such tasks requires choosing a {\it prompt}, a textual description of the task. The choice of the prompt can have a high impact on performance; as a result, researchers and practitioners spend considerable effort on experimenting with various prompts. Further, a good prompt for one domain may be less suitable for another domain, which requires matching or tailoring prompts to each domain. We now show how our evaluation techniques enable automating the prompt selection process.

We consider a summarization task where a long piece of text (example shown in Figure~\ref{fig:tourist_lms_long_text_example}) must be summarized into a short piece of text capturing the most pertinent information of the original text. We create a set of $t$ tasks, where each task consists of $n$ pieces of text that must be summarized. To evaluate the quality of the summary, each piece of text $r_i$ is associated with $q$ multiple choice question that must be answered given the summary. Importantly, these questions are only available at evaluation time, so the process generating the summary has no access to the evaluation questions when generating the summary. 

\begin{figure*}[t]
\begin{center}
\begin{tabular}{| l |} 
 \hline
 The city of Alton is a small, peaceful village in the middle of the country. It has a population\\
 of around 20,000 people. The village has a lot of art galleries, and there are many museums.\\
 Many people from other cities visit Alton in order to enjoy the artistic and cultural life.\\
 The main attraction of the city is the maze of mirrors. \\
 \hline
 Question 1: What is Alton famous for?  \\
 a. Its location  b. Its art and culture  c. The underground caves  d. The ocean \\
Answer 1: b. Its art and culture \\

Question 2: What does Alton have many of?  a. Schools  b. Museums  c. Factories  d. Cars \\
Answer 2: b. Museums \\

Question 3: What is the main attraction in Alton?  a. the underground caves  b. the maze of\\
mirrors  c. the jungle  d. the ocean \\
Answer 3: b. the maze of mirrors \\
 \hline
\end{tabular}
\end{center}
\caption{Example tourist destination page (long text). \label{fig:tourist_lms_long_text_example}}
\end{figure*}

More concretely, the LLM task description is depicted in Figure~\ref{fig:tourist_lms_task_desc}.

\begin{figure*}[t]
\begin{center}
\begin{tabular}{| l |} 

 \hline
A wise and intelligent assistant is summarizing entries regarding tourist destinations, trying to\\
generate shorter summaries which only include the most important information. \\
<specific prompt> \\
Here are example entries and summaries. \\
Entry 1: <example full text1> \\
Summary 1: <example summary 1> \\
Entry 2: <example full text 2> \\
Summary 2: <example summary 2> \\
Entry 3: \\
\hline
\end{tabular}
\end{center}
\caption{Example task description for the LLM. \label{fig:tourist_lms_task_desc}}
\end{figure*}

In the template shown in Figure~\ref{fig:tourist_lms_task_desc}, $<$example full text$_i$$>$ and $<$example summary $_i$$>$ are manually created full text descriptions of a destination and a manually created summary text for that destination. The prompts to be evaluated are injected into the above template in the $<$specific prompt$>$ location. In our evaluation we consider the following prompts (clearly many other prompts can be evaluated in a similar way). 

\begin{figure*}[t]
\begin{center}
\begin{tabular}{|l|} 
\hline
1. The assistant keeps almost the entire original text, changing very few words.\\

2. The assistant keeps only the most important information, feeling free to remove boring parts.\\

3. The assistant doesn't remove any of the information, and keeps pretty much all the details.\\

4. The assistant keeps any information that may be interesting for tourists, but skips the.\\
~~~~~remaining parts.\\
5. The assistant sticks to the facts in the summary, and ignores whatever is subjective evaluation.\\

6. The assistant uses simple words.\\

7. The assistant tends to use complex and impressive language.\\
\hline
\end{tabular}
\end{center}
\caption{Prompts (agents) used in prompt evaluation for summarization. \label{fig:prompts}}
\end{figure*}

Our analysis is based on pieces of text about made-up tourist destinations, giving a description of the location as well as tourist attractions available there. The evaluation questions relate to the description of the destination and the attractions. We ensure that details pertain to made-up destinations so that no background knowledge that the LLM may have can be used to answer the questions regarding the destination. 
The prompts we use are show in Figure~\ref{fig:prompts}.

\begin{figure*}[t]
\begin{center}
\hspace{-2.5cm}
\includegraphics[scale=0.8]{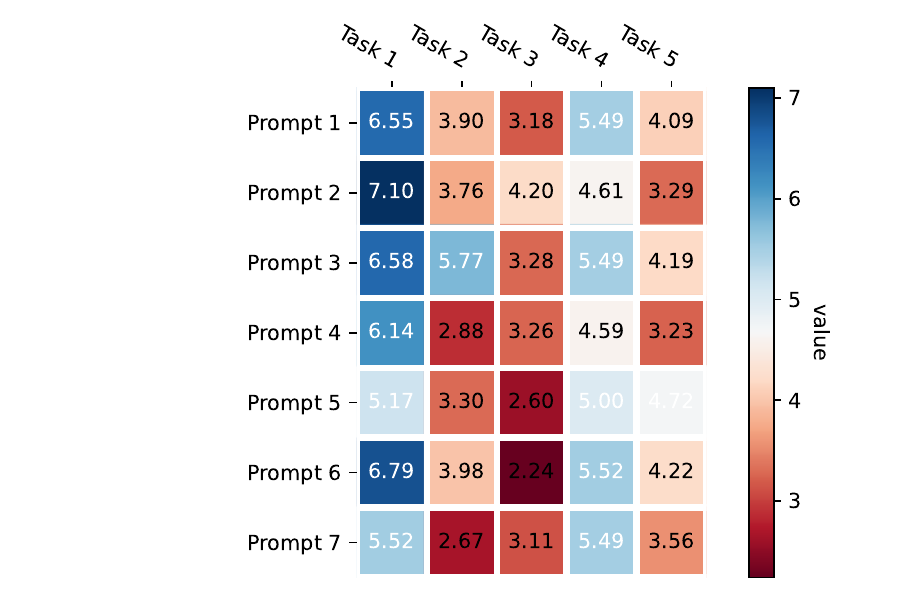}
\caption{Agent vs task matrix: scores achieved for different prompts across tasks. \label{fig:tourist_lms}}
\end{center}
\end{figure*}

An LLM performs the summarization tasks given guidance provided by a textual prompt. We then use our VasE to select the best prompt for the task at hand; in this domain then, a specific prompt corresponds to an agent in the AvT setting. Our score for a generated summary depends on both the number of questions that can be correctly answered given the summary, and the length of the summary. In our evaluation, we have $t=5$ tasks, each consisting of $n=3$ pieces of text to be summarized, each including $q=3$ questions about the text (i.e. $n=3$ tourist destinations per task, each associated with $q=3$ questions, so each task is scored based on the summary length and the $n \cdot q=9$ questions for the destinations of the task). The score achieved on a piece of text is the number of questions that can be answered given the summary minus a length penalty term (in our evaluation we use a penalty of $0.1 \log (\text{length})$, though clearly many other evaluation options may be used). The total score on a task is the sum of the scores on the $n$ pieces of text that the task is composed of. 
We generate the summary of a piece of text by applying an LLM.
We apply few-shot prompting to generate the summary. In all cases we use the same few-shot examples of texts and summaries, and we consider various prompts by adding a short description.

The matrix of results is shown in Figure~\ref{fig:tourist_lms}: this leads to 5 votes among 7 candidates. Despite the closeness in values, VasE identifies a strong Condorcet winner: prompt 6. VasE results are shown in Table~\ref{tab:tourist_lms_full}.

\begin{table*}[ht!]
\begin{center}
\begin{tabular}{|c|ll|}
\multicolumn{3}{c}{\bf Approval(k=2)}\\
\hline
Rank & Agent & Score\\
\hline
1 & \textsc{6} & 4\\
2 & \textsc{3} & 2\\
3 & \textsc{2} & 2\\
4 & \textsc{7} & 1\\
5 & \textsc{5} & 1\\
6 & \textsc{4} & 0\\
7 & \textsc{1} & 0\\
\hline
\end{tabular}
~~~~~
\begin{tabular}{|c|ll|}
\multicolumn{3}{c}{\bf Borda}\\
\hline
Rank & Agent & Score\\
\hline
1 & \textsc{3} & 22\\
2 & \textsc{6} & 21\\
3 & \textsc{2} & 17\\
4 & \textsc{1} & 17\\
5 & \textsc{5} & 11\\
6 & \textsc{7} & 10\\
7 & \textsc{4} & 7\\
\hline
\end{tabular}
~~~~~
\begin{tabular}{|c|ll|}
\multicolumn{3}{c}{\bf Copeland}\\
\hline
Rank & Agent & Score\\
\hline
1 & \textsc{6} & 6.0\\
2 & \textsc{3} & 5.0\\
3 & \textsc{1} & 4.0\\
4 & \textsc{2} & 3.0\\
5 & \textsc{7} & 1.0\\
6 & \textsc{5} & 1.0\\
7 & \textsc{4} & 1.0\\
\hline
\end{tabular}\\

\vspace{0.5cm}

\begin{tabular}{|c|ll|}
\multicolumn{3}{c}{\bf Kemeny-Young}\\
\hline
Rank & Agent & Score\\
\hline
1 & \textsc{6} & 21\\
2 & \textsc{3} & 20\\
3 & \textsc{1} & 15\\
4 & \textsc{2} & 11\\
5 & \textsc{4} & 5\\
6 & \textsc{7} & 3\\
7 & \textsc{5} & 0\\
\hline
\end{tabular}
~~~
\begin{tabular}{|c|ll|}
\multicolumn{3}{c}{\bf Iterative Maximal Lotteries}\\
\hline
Rank & Agent & Score\\
\hline
1 & \textsc{6} & 5.00\\
2 & \textsc{3} & 4.00\\
3 & \textsc{1} & 3.00\\
4 & \textsc{2} & 2.00\\
5 & \textsc{4} & 0.33\\
6 & \textsc{7} & 0.33\\
7 & \textsc{5} & 0.33\\
\hline
\end{tabular}
~~~
\begin{tabular}{|c|ll|}
\multicolumn{3}{c}{\bf Plurality}\\
\hline
Rank & Agent & Score\\
\hline
1 & \textsc{2} & 2\\
2 & \textsc{6} & 1\\
3 & \textsc{5} & 1\\
4 & \textsc{3} & 1\\
5 & \textsc{7} & 0\\
6 & \textsc{4} & 0\\
7 & \textsc{1} & 0\\
\hline
\end{tabular}\\

\vspace{0.5cm}

\begin{tabular}{|c|ll|}
\multicolumn{3}{c}{\bf Ranked Pairs}\\
\hline
Rank & Agent & Score\\
\hline
1 & \textsc{6} & 45\\
2 & \textsc{3} & 33\\
3 & \textsc{1} & 18\\
4 & \textsc{2} & 8\\
5 & \textsc{7} & 1\\
6 & \textsc{5} & 1\\
7 & \textsc{4} & 0\\
\hline
\end{tabular}
~~~
\begin{tabular}{|c|ll|}
\multicolumn{3}{c}{\bf Schulze}\\
\hline
Rank & Agent & Score\\
\hline
1 & \textsc{6} & 19\\
2 & \textsc{3} & 16\\
3 & \textsc{1} & 12\\
4 & \textsc{2} & 9\\
5 & \textsc{4} & 4\\
6 & \textsc{5} & 2\\
7 & \textsc{7} & 0\\
\hline
\end{tabular}
~~~
\begin{tabular}{|c|ll|}
\multicolumn{3}{c}{\bf STV(num winners = 3)}\\
\hline
Rank & Agent & Score\\
\hline
1 & \textsc{2} & 14.2\\
2 & \textsc{6} & 13.2\\
3 & \textsc{5} & 7.1\\
4 & \textsc{3} & 6.1\\
5 & \textsc{7} & 5.0\\
6 & \textsc{4} & 4.0\\
7 & \textsc{1} & 3.0\\
\hline
\end{tabular}
\end{center}
\caption{VasE results on the prompt engineering for text summarization.\label{tab:tourist_lms_full}}
\end{table*}

All of the Condorcet methods (Copeland, Kemeny-Young, Maximal Lotteries, Schulze, Ranked Pairs) choose the same decisive order for the top four prompts: $6 \succ 3 \succ 1 \succ 2$. The scoring methods vary due to the top-ranked agents differing significantly across tasks, electing $6$, $3$, $2$, $2$ for Approval($k=2)$, Borda, Plurality, and STV, respectively. 

Similarly to Atari, when applying Nash averaging the task player mixes over only two tasks $(T3: 0.56, T5: 0.44)$ resulting in first place tie between prompts 2 and 3, and full rank $2 = 3 \succ 5 \succ 1 \succ 7 \succ 6$. In this case, adversarial task selection results in the strong Condorcet winner being ranked last place by Nash averaging.


\section{Application to HELM}

\subsection{Application to the HELM: Core Scenarios}
\label{app:helm}

We applied VasE to the data from the central repository Holistic Evaluation of Language Models (HELM)~\citep{liang2022holistic}\footnote{Note: the HELM data processed here was retrieved in April 2023.}.\\

Here, we describe the application of VasE to data from a central repository curated and maintained by Stanford's Center for Research on Foundation Models called Holistic Evaluation of Language Models (HELM). In this data source, there are 39 models (agents) evaluated across a 42 different {\it scenarios} in several categories such as question \& answering, information retrieval, toxicity detection, sentiment analysis, knowledge, reasoning, etc., as well as several metrics and taxonomies for analyses.


In this section, we restrict ourselves to the core scenarios~\citep[Section 3]{liang2022holistic}, a broad subset of all scenarios: 37 models, and tasks each correspond to a specific metric (accuracy, calibration error, robustness, fairness, efficiency, bias, toxicity, and summarization metrics) upon which models are judged. For each metric (task), models are sorted by head-to-head win rate across all scenarios pertinent for that metric. VasE identifies {\tt text-davinci-002} as a strong Condorcet winner and is top-ranked by all Condorcet methods as well as Borda, with {\tt text-davinci-003}, {\tt Cohere Command beta (52.4B)}, {\tt TNLG v2 (530B)}, {\tt OPT (175B)}, and {\tt text-ada-001} appearing in the top five models across methods. 

We now highlight customized evaluation using non-uniform task weightings. Suppose, instead of an overall aggregate measure where all metrics are treated equally, the user strongly prefers nontoxic, fair and unbiased models. VasE can produce a cautious choice by counting \eg the toxicity vote five times, and bias and fairness each three times. In this case, there is no longer any Condorcet winners, and {\tt YaLM (100B)} is top-ranked by 7 of 9 voting methods. This is better than simply choosing the best model under the one metric because preferences of other metrics are still accounted for in the final rank. As another example: in the HELM paper, authors cautioned against looking at efficiency in isolation (independent of accuracy). So, VasE can choose an efficient model by counting efficiency six times, accuracy four times, and each of other metrics once. In this case, there is a weak Condorcet winner: {\tt text-curie-001}, while App($k=5$), plurality, and STV now elects {\tt text-ada-001}.



In this case, VasE identifies {\tt text-davinci-002} as a weak Condorcet winner.
The results for the top-10 agents are shown in Tables \ref{tab:helm-core-full1}, \ref{tab:helm-core-full2}, and \ref{tab:helm-core-full3}.

\begin{table*}[ht!]
\begin{center}
\begin{tabular}{|c|ll|}
\multicolumn{3}{c}{\bf Approval(k=5)}\\
\hline
Rank & Agent & Score\\
\hline
1 & {\tt text-davinci-003} & 4\\
2 & {\tt Cohere Command beta (52.4B)} & 4\\
3 & {\tt text-davinci-002} & 3\\
4 & {\tt TNLG v2 (530B)} & 3\\
5 & {\tt Anthropic-LM v4-s3 (52B)} & 3\\
6 & {\tt YaLM (100B)} & 2\\
7 & {\tt text-ada-001} & 2\\
8 & {\tt text-babbage-001} & 2\\
9 & {\tt text-curie-001} & 2\\
10 & {\tt ada (350M)} & 2\\
\hline
\end{tabular}\\ 

\vspace{0.5cm}

\begin{tabular}{|c|ll|}
\multicolumn{3}{c}{\bf Borda}\\
\hline
Rank & Agent & Score\\
\hline
1 & {\tt text-davinci-002} & 232\\
2 & {\tt text-davinci-003} & 213\\
3 & {\tt TNLG v2 (530B)} & 191\\
4 & {\tt text-curie-001} & 189\\
5 & {\tt davinci (175B)} & 182\\
6 & {\tt OPT (175B)} & 180\\
7 & {\tt GPT-J (6B)} & 176\\
8 & {\tt OPT (66B)} & 173\\
9 & {\tt Cohere Command beta (52.4B)} & 173\\
10 & {\tt Cohere Command beta (6.1B)} & 172\\
\hline
\end{tabular}\\

\vspace{0.5cm}

\begin{tabular}{|c|ll|}
\multicolumn{3}{c}{\bf Copeland}\\
\hline
Rank & Agent & Score\\
\hline
1 & {\tt text-davinci-002} & 35.5\\
2 & {\tt text-davinci-003} & 33.0\\
3 & {\tt TNLG v2 (530B)} & 31.5\\
4 & {\tt text-curie-001} & 28.5\\
5 & {\tt OPT (175B)} & 28.0\\
6 & {\tt Cohere Command beta (6.1B)} & 28.0\\
7 & {\tt Cohere Command beta (52.4B)} & 27.5\\
8 & {\tt davinci (175B)} & 27.0\\
9 & {\tt J1-Grande v2 beta (17B)} & 26.0\\
10 & {\tt GPT-J (6B)} & 24.5\\
\hline
\end{tabular}
\end{center}
\caption{VasE methods Approval, Borda, and Copeland on HELM Core Scenarios. \label{tab:helm-core-full1}}
\end{table*}

\begin{table*}
\begin{center}
\begin{tabular}{|c|ll|}
\multicolumn{3}{c}{\bf Iterative Maximal Lotteries}\\
\hline
Rank & Agent & Score\\
\hline
1 & {\tt text-davinci-002} & 14.64\\
2 & {\tt text-curie-001} & 14.36\\
3 & {\tt text-davinci-003} & 13.69\\
4 & {\tt OPT (175B)} & 13.31\\
5 & {\tt Cohere Command beta (52.4B)} & 12.53\\
6 & {\tt GPT-J (6B)} & 12.27\\
7 & {\tt Cohere xlarge v20221108 (52} & 12.15\\
8 & {\tt GPT-NeoX (20B)} & 12.05\\
9 & {\tt TNLG v2 (530B)} & 11.69\\
10 & {\tt YaLM (100B)} & 11.17\\
\hline
\end{tabular}\\

\vspace{0.5cm}

\begin{tabular}{|c|ll|}
\multicolumn{3}{c}{\bf Plurality}\\
\hline
Rank & Agent & Score\\
\hline
1 & {\tt text-ada-001} & 2\\
2 & {\tt Cohere Command beta (52.4B)} & 2\\
3 & {\tt YaLM (100B)} & 1\\
4 & {\tt text-davinci-002} & 1\\
5 & {\tt curie (6.7B)} & 1\\
6 & {\tt TNLG v2 (530B)} & 1\\
7 & {\tt GLM (130B)} & 0\\
8 & {\tt text-babbage-001} & 0\\
9 & {\tt text-curie-001} & 0\\
10 & {\tt text-davinci-003} & 0\\
\hline
\end{tabular}\\

\vspace{0.5cm}

\begin{tabular}{|c|ll|}
\multicolumn{3}{c}{\bf Ranked Pairs}\\
\hline
Rank & Agent & Score\\
\hline
1 & {\tt text-davinci-002} & 1640\\
2 & {\tt text-davinci-003} & 1464\\
3 & {\tt Cohere Command beta (52.4B)} & 1324\\
4 & {\tt TNLG v2 (530B)} & 1262\\
5 & {\tt J1-Grande v2 beta (17B)} & 1158\\
6 & {\tt Cohere Command beta (6.1B)} & 1102\\
7 & {\tt GLM (130B)} & 1026\\
8 & {\tt text-curie-001} & 976\\
9 & {\tt OPT (175B)} & 882\\
10 & {\tt davinci (175B)} & 794\\
\hline
\end{tabular}
\end{center}
\caption{VasE methods IML, Plurality, and Ranked Pairs on HELM Core Scenarios. \label{tab:helm-core-full2}}
\end{table*}

\begin{table*}
\begin{center}
\begin{tabular}{|c|ll|}
\multicolumn{3}{c}{\bf Schulze}\\
\hline
Rank & Agent & Score\\
\hline
1 & {\tt text-davinci-002} & 166\\
2 & {\tt text-davinci-003} & 161\\
3 & {\tt Cohere Command beta (52.4B)} & 156\\
4 & {\tt TNLG v2 (530B)} & 151\\
5 & {\tt J1-Grande v2 beta (17B)} & 145\\
6 & {\tt Luminous Supreme (70B)} & 139\\
7 & {\tt Anthropic-LM v4-s3 (52B)} & 135\\
8 & {\tt BLOOM (176B)} & 131\\
9 & {\tt Cohere xlarge v20221108} & 128\\
10 & {\tt Cohere Command beta (6.1B)} & 125\\
\hline
\end{tabular}\\

\vspace{0.5cm}

\begin{tabular}{|c|ll|}
\multicolumn{3}{c}{\bf STV(num winners = 5)}\\
\hline
Rank & Agent & Score\\
\hline
1 & {\tt text-ada-001} & 74.2\\
2 & {\tt Cohere Command beta (52.4B)} & 73.2\\
3 & {\tt text-davinci-002} & 72.2\\
4 & {\tt YaLM (100B)} & 71.2\\
5 & {\tt curie (6.7B)} & 37.1\\
6 & {\tt TNLG v2 (530B)} & 36.1\\
7 & {\tt GLM (130B)} & 35.0\\
8 & {\tt text-babbage-001} & 34.0\\
9 & {\tt text-curie-001} & 33.0\\
10 & {\tt text-davinci-003} & 32.0\\
\hline
\end{tabular}
\end{center}
\caption{VasE methods Schulze and STV on HELM Core Scenarios. \label{tab:helm-core-full3}}
\end{table*}

\subsection{VasE Applied to All HELM Scenarios}

Our second application of VasE to the HELM data involves $m = 37$ models over $n = 39$ tasks, each corresponding to a scenario, where models are compared via scenario-specific metrics. 

This is a quite different than the application to the core scenarios where the scenarios are carefully selected subset of all scenarios and head-to-head win rates are already provided. Second: the tasks now correspond to all the scenarios available from the results repository (\url{https://crfm.stanford.edu/helm/latest/?groups=1}).

The scenarios and their corresponding domain-specific metric are shown in Table~\ref{tab:helm-metrics}. For this application, we assemble the votes from the first table of each scenario listed on each scenario's page (\eg BoolQ's data is available at \url{https://crfm.stanford.edu/helm/latest/?group=boolq} and the data from first table on the page is taken from \url{https://storage.googleapis.com/crfm-helm-public/benchmark_output/runs/v0.2.2/groups/boolq.json}).

\begin{table}[ht!]
    \centering
    \begin{tabular}{|l|l|}
    \hline
    {\bf Metric}  & {\bf Scenario}  \\
    \hline
    EM   &  boolq \\
         &  hellaswag\\ 
         &  openbookqa\\ 
         &  truthful\_qa\\ 
         &  mmlu\\ 
         &  imdb\\ 
         &  raft\\ 
         &  civil\_comments\\ 
         &  blimp\\ 
         &  wikifact\\ 
         &  babi\_qa\\ 
         &  dyck\_language\\ 
         &  synthetic\_reasoning\\ 
         &  gsm\\ 
         &  legal\_support\\ 
         &  lsat\_qa\\ 
         &  legal\_support\\ 
         &  entity\_data\_imputation\\ 
         &  entity\_matching\\ 
         &  bbq\\ 
    \hline
    F1   & narrative\_qa\\
         & natural\_qa\_closed\_book\\
         & natural\_qa\_openbook\_longans\\
         & quac \\
         & synthetic\_reasoning\_natural\\
    \hline
    RR@10 & msmarco\_regular \\
    \hline
    NDCG@10 & msmarco\_trec \\
    \hline
    ROUGE-2 & summarization\_cnndm\\
            & summarization\_xsum\\
    \hline
    BPB   & ice\\
          & the\_pile\\
          & twitter\_aae\\
          & twitter\_aae\_white\\
    \hline
    Equivalent & math\_regular\\
               & math\_chain\_of\_thought\\
    \hline
    LCS        & copyright\_text\\
    \hline
    Stereotypes (race)  & disinformation\_reiteration\\
    \hline
    Toxic fraction & bold\\
                   & real\_toxicity\_prompts\\
    \hline
    \end{tabular}
    \caption{List of all HELM scenarios categorized by metric used to rank models.}
    \label{tab:helm-metrics}
\end{table}

Several features of VasE are showcased in this second application.
First, the metrics all have numerical values which are {\it not} comparable across tasks.
Second, not all models have been evaluated on each scenario. 
This leads to votes of slightly variable length, but not \eg to arbitrary value choices for missing scores.
In this case, VasE identifies {\tt code-davinci-002} as a strong Condorcet winner. The top-10 models for each method are shown in Tables \ref{tab:helm-all-full1}, \ref{tab:helm-all-full2}, and \ref{tab:helm-all-full3}.

\begin{table*}
\begin{center}
\begin{tabular}{|c|ll|}
\multicolumn{3}{c}{\bf Approval(k=5)}\\
\hline
Rank & Agent & Score\\
\hline
1 & {\tt text-davinci-003} & 28\\
2 & {\tt text-davinci-002} & 23\\
3 & {\tt Cohere Command beta (52B)} & 17\\
4 & {\tt TNLG v2 (530B)} & 16\\
5 & {\tt Anthropic-LM v4-s3 (52B)} & 16\\
6 & {\tt code-davinci-002} & 7\\
7 & {\tt OPT (175B)} & 6\\
8 & {\tt J1-Grande v2 beta (17B)} & 6\\
9 & {\tt GPT-NeoX (20B)} & 6\\
10 & {\tt BLOOM (176B)} & 6\\
\hline
\end{tabular}\\

\vspace{0.5cm}

\begin{tabular}{|c|ll|}
\multicolumn{3}{c}{\bf Borda}\\
\hline
Rank & Agent & Score\\
\hline
1 & {\tt text-davinci-003} & 1079\\
2 & {\tt text-davinci-002} & 1057\\
3 & {\tt TNLG v2 (530B)} & 963\\
4 & {\tt Cohere Command beta (52)} & 945\\
5 & {\tt Anthropic-LM v4-s3 (52B)} & 917\\
6 & {\tt J1-Grande v2 beta (17B)} & 905\\
7 & {\tt Cohere xlarge v20221108 52B} & 811\\
8 & {\tt Cohere xlarge v20220609 52B} & 796\\
9 & {\tt OPT (175B)} & 748\\
10 & {\tt davinci (175B)} & 747\\
\hline
\end{tabular}\\

\vspace{0.5cm}

\begin{tabular}{|c|ll|}
\multicolumn{3}{c}{\bf Copeland}\\
\hline
Rank & Agent & Score\\
\hline
1 & {\tt code-davinci-002} & 38.0\\
2 & {\tt text-davinci-003} & 37.0\\
3 & {\tt text-davinci-002} & 36.0\\
4 & {\tt Cohere Command beta (52.4B)} & 34.5\\
5 & {\tt Anthropic-LM v4-s3 (52B)} & 34.5\\
6 & {\tt TNLG v2 (530B)} & 33.0\\
7 & {\tt J1-Grande v2 beta (17B)} & 32.0\\
8 & {\tt Luminous Supreme (70B)} & 31.0\\
9 & {\tt Cohere xlarge v20221108 52B} & 30.0\\
10 & {\tt Cohere xlarge v20220609 52B} & 29.0\\
\hline
\end{tabular}
\end{center}
\caption{VasE results for Approval, Borda, and Copeland methods on HELM All Scenarios. \label{tab:helm-all-full1}}
\end{table*}

\begin{table*}
\begin{center}
\begin{tabular}{|c|ll|}
\multicolumn{3}{c}{\bf Iterative Maximal Lotteries}\\
\hline
Rank & Agent & Score\\
\hline
1 & {\tt code-davinci-002} & 27.00\\
2 & {\tt text-davinci-003} & 26.00\\
3 & {\tt text-davinci-002} & 25.00\\
4 & {\tt Cohere Command beta (52.4B)} & 23.52\\
5 & {\tt Anthropic-LM v4-s3 (52B)} & 23.48\\
6 & {\tt TNLG v2 (530B)} & 23.00\\
7 & {\tt J1-Grande v2 beta (17B)} & 22.00\\
8 & {\tt Luminous Supreme (70B)} & 21.00\\
9 & {\tt Cohere xlarge v20221108 52B} & 20.00\\
10 & {\tt Cohere xlarge v20220609 52B} & 19.00\\
\hline
\end{tabular}\\

\vspace{0.5cm}

\begin{tabular}{|c|ll|}
\multicolumn{3}{c}{\bf Plurality}\\
\hline
Rank & Agent & Score\\
\hline
1 & {\tt text-davinci-003} & 11\\
2 & {\tt text-davinci-002} & 5\\
3 & {\tt code-davinci-002} & 5\\
4 & {\tt Cohere Command beta 52B} & 4\\
5 & {\tt OPT (175B)} & 2\\
6 & {\tt Anthropic-LM v4-s3 52B} & 2\\
7 & {\tt davinci (175B)} & 1\\
8 & {\tt babbage (1.3B)} & 1\\
9 & {\tt TNLG v2 (530B)} & 1\\
10 & {\tt T0pp (11B)} & 1\\
\hline
\end{tabular} \\

\vspace{0.5cm}

\begin{tabular}{|c|ll|}
\multicolumn{3}{c}{\bf Ranked Pairs}\\
\hline
Rank & Agent & Score\\
\hline
1 & {\tt code-davinci-002} & 11544\\
2 & {\tt text-davinci-003} & 11395\\
3 & {\tt text-davinci-002} & 10472\\
4 & {\tt Cohere Command beta 52B} & 9581\\
5 & {\tt Anthropic-LM v4-s3 52B} & 8888\\
6 & {\tt TNLG v2 (530B)} & 8202\\
7 & {\tt J1-Grande v2 beta (17B)} & 7451\\
8 & {\tt Luminous Supreme (70B)} & 6784\\
9 & {\tt Cohere xlarge v20221108 52B} & 6258\\
10 & {\tt Cohere xlarge v20220609 52B} & 5748\\
\hline
\end{tabular}
\end{center}
\caption{VasE results for IML, Plurality, and Ranked Pairs methods on HELM All Scenarios. \label{tab:helm-all-full2}}
\end{table*}

\begin{table*}
\begin{center}
\begin{tabular}{|c|ll|}
\multicolumn{3}{c}{\bf Schulze}\\
\hline
Rank & Agent & Score\\
\hline
1 & {\tt code-davinci-002} & 681\\
2 & {\tt text-davinci-003} & 676\\
3 & {\tt text-davinci-002} & 653\\
4 & {\tt Anthropic-LM v4-s3 (52B)} & 627\\
5 & {\tt Cohere Command beta 52B} & 610\\
6 & {\tt TNLG v2 (530B)} & 591\\
7 & {\tt J1-Grande v2 beta (17B)} & 566\\
8 & {\tt Luminous Supreme (70B)} & 551\\
9 & {\tt Cohere xlarge v20221108 52B} & 535\\
10 & {\tt Cohere xlarge v20220609 52B} & 514\\
\hline
\end{tabular}\\

\vspace{0.5cm}

\begin{tabular}{|c|ll|}
\multicolumn{3}{c}{\bf STV(num winners = 5)}\\
\hline
Rank & Agent & Score\\
\hline
1 & {\tt text-davinci-003} & 78.11\\
2 & {\tt text-davinci-002} & 77.7\\
3 & {\tt Cohere Command beta (52.4B)} & 76.7\\
4 & {\tt code-davinci-002} & 75.7\\
5 & {\tt OPT (175B)} & 74.9\\
6 & {\tt babbage (1.3B)} & 39.3\\
7 & {\tt Anthropic-LM v4-s3 (52B)} & 38.2\\
8 & {\tt davinci (175B)} & 37.1\\
9 & {\tt TNLG v2 (530B)} & 36.1\\
10 & {\tt T0pp (11B)} & 35.1\\
\hline
\end{tabular}
\end{center}
\caption{VasE results for Schulze and STV methods on HELM All Scenarios. \label{tab:helm-all-full3}}
\end{table*}

\end{document}